\documentclass[format=acmsmall, review=false, screen=true]{acmart}
\usepackage{changes}
\usepackage{subfigure}
\usepackage{mathptmx}
\usepackage{amsmath,times}
\usepackage{amssymb}
\usepackage[mathcal]{eucal}
\usepackage{url}
\usepackage{psfrag}
\usepackage{booktabs} 
\usepackage{graphicx}
\usepackage{epstopdf}

\newcommand{\nn}{\nonumber}

\newcommand{\YueE}[1]{{\color{black}{\bf}#1}}

\newcommand{\compose}{\mathbin{\vert\vert\vert}}
\newcommand{\trans}[1]{\overset{#1}\rightarrow}

\newtheorem{assumption}[theorem]{Assumption}
\newtheorem{remark}{Remark}

\usepackage[ruled]{algorithm2e}

\SetAlFnt{\small}
\SetAlCapFnt{\small}
\SetAlCapNameFnt{\small}
\SetAlCapHSkip{0pt}
\IncMargin{-\parindent}

\acmJournal{TiiS}
\acmVolume{9}
\acmNumber{4}
\acmArticle{39}
\acmYear{2017}
\acmMonth{3}
\copyrightyear{2017}

\setcopyright{acmlicensed}

\acmDOI{0000001.0000001}

\received{December 2016}
\received[revised]{May 2017}
\received[accepted]{May 2017}

\begin{document}

\markboth{Y. Wang, L. R. Humphrey, Z. Liao, and H. Zheng}{Trust-based Multi-Robot Symbolic Motion Planning with a Human-in-the-Loop}

\title[Trust-based Multi-Robot Symbolic Motion Planning with a Human-in-the-Loop]{Trust-based Multi-Robot Symbolic Motion Planning with a Human-in-the-Loop}  
\author{Yue Wang}
\orcid{0000-0003-0146-7262}
\affiliation{%
  \institution{Clemson University}
  \department{Department of Mechanical Engineering}
    \streetaddress{237 Fluor Daniel Building}
  \city{Clemson}
  \state{SC}
  \postcode{29630}
  \country{USA}}
\author{Laura R. Humphrey}
\affiliation{%
  \institution{Air Force Research Laboratory}
  \department{Control Science Center of Excellence}
  \city{WFAFB}
  \state{OH}
  \postcode{45433}
  \country{USA}
}
\author{Zhanrui Liao}
\affiliation{%
  \institution{Clemson University}
  \department{Department of Mechanical Engineering}
    \streetaddress{237 Fluor Daniel Building}
  \city{Clemson}
  \state{SC}
  \postcode{29630}
  \country{USA}}
\author{Huanfei Zheng}
\affiliation{%
  \institution{Clemson University}
  \department{Department of Mechanical Engineering}
    \streetaddress{237 Fluor Daniel Building}
  \city{Clemson}
  \state{SC}
  \postcode{29630}
  \country{USA}}

\begin{abstract}
	Symbolic motion planning for robots is the process of specifying and planning robot tasks in a discrete space, then carrying them out in a continuous space in a manner that preserves the discrete-level task specifications. Despite progress in symbolic motion planning, many challenges remain, including addressing scalability for multi-robot systems and improving solutions by incorporating human intelligence. In this paper, distributed symbolic motion planning for multi-robot systems is developed to address scalability. More specifically, compositional reasoning approaches are developed to decompose the global planning problem, and atomic propositions for observation, communication, and control are proposed to address inter-robot collision avoidance. To improve solution quality and adaptability, a dynamic, quantitative, and probabilistic human-to-robot trust model is developed to aid this decomposition. Furthermore, a trust-based real-time switching framework is proposed to switch between autonomous and manual motion planning for tradeoffs between task safety and efficiency. Deadlock- and livelock-free algorithms are designed to guarantee reachability of goals with a human-in-the-loop. A set of non-trivial multi-robot simulations with direct human input and trust evaluation are provided demonstrating the successful implementation of the trust-based multi-robot symbolic motion planning methods. 
\end{abstract}

%


%
%



\thanks{This work is supported by the Air Force Research Laboratory's Summer Faculty Fellowship Program (SFFP) and the Air Force Office of Scientific Research Young Investigator Program under grant no. FA9550-17-1-0050.
	
	Author's addresses: Y. Wang, Z. Liao, and H. Zheng, Department of Mechanical Engineering, Clemson University, Clemson SC 39634; L. R. Humphrey, Wright Patterson Air Force Research Laboratory, OH 45433.
}

\maketitle

\section{Introduction}

Despite advances in autonomy for robotic systems, human supervision/collaboration is often still necessary to ensure safe and efficient operations in uncertain, dynamic, or noisy environments where robot sensing and perception may not be fully reliable. While ideally autonomous robots are expected to be self-sufficient, there are practical tradeoffs between cost and performance. On one hand, humans excel at high-level decision-making in such environments and can help autonomous robots achieve better performance while keeping design costs low. On the other hand, human error is a main cause of machine malfunctions~\cite{rouse1983human,adams2002hri}, and human performance degrades when overloaded~\cite{dhillon1997safety,crandall2005validating}. When designing autonomous robotic systems, it is therefore important to consider factors related to human-robot interaction (HRI) \cite{GoSc-HCI-07}. However, although much research has been conducted on the development of effective approaches for HRI, extant solutions remain highly specialized and focused on human-machine interface (HMI) design~\cite{BoBaSi-TSMC-13}. The modeling, analysis, and implementation of effective HRI remains largely an open problem~\cite{hayes2013challenges}. The current design process for robotic systems, especially in high-level decision-making and coordination, is still largely one of trial and error. In particular, the process often lacks quantitative models and real-time analytic approaches that could be used to provide safety and performance guarantees. The HRI problem in which a single human must interact with multiple autonomous robots is especially challenging due to the problem size, the need for robot coordination, the possibility of unintended emergent
behaviors, etc.

Although much work has been undertaken to characterize physical HRI (pHRI) in terms of safety, performance, adaptability, etc.~\cite{ikemoto2012physical,haddadin2011safe}, an important factor to consider with respect to social HRI (sHRI) and cognitive HRI (cHRI) is human trust in autonomous robots. Establishing trust in robots is the bottleneck in the development and integration of HRI systems. Trust can be defined as~\cite{lee2004trust}  {(page 51)}
\begin{quote}
	``the attitude that an agent will help achieve an individual's
goals in a situation characterized by uncertainty and vulnerability."
\end{quote} 
Trust is a dynamic feature of HRI~\cite{Levin2006trust} that heavily affects a human's acceptance and hence use of a robot \cite{wagner2009role,hancock2011meta}. Humans respond socially to robots, establishing a level of trust to manage workload not possible with mere human endeavor~\cite{freedy2007measurement,hoff2015trust}. Without trust, humans are likely to ignore or dismiss the assistance of robots and opt to complete a task on their own. Informed trust is an accurate assessment of when and how much autonomy should be employed, and when to intervene. Both economic and ergonomic benefits can be obtained by properly trusting the reliability of autonomous robots. Humans can either gain or lose trust in robots based on the progress of the task~\cite{crandall2003towards}. Moreover, consideration of trust is especially important for the supervisory control of multiple robots, since the tasks must be carefully allocated to ensure time-critical issues are addressed while human workload is kept within acceptable bounds \cite{BaHaKiSc-ROMAN-08,ruff2002human}. 

Some progress in addressing these deficiencies has been made through the application of formal methods -- 
i.e., mathematically-based tools and techniques for system specification, 
design, and verification \cite{baier2008principles} -- 
to problems involving HRI \cite{BoBaSi-TSMC-13,humphrey2014formal}. The temporal logics commonly used
in formal methods provide a high-level, human-like language for specifying desired properties or behaviors of a
system, which can then be used to either verify or synthesize provably correct designs~\cite{HuRy-Book-12,baier2008principles}.
Formal methods have been utilized in a wide range of applications including verification of HMI designs \cite{BoBaSi-TSMC-13} and implementation of robot motion planning and control, leading to approaches such as symbolic motion planning for autonomous robots \cite{belta2007symbolic}. 

 
{Symbolic motion planning extends traditional motion planning -- generally concerned with fast computation of efficient, optimal, and dynamically feasible paths that allow mobile robots to reach a set of goal locations while avoiding obstacles \cite{algorithmssteven} -- by also incorporating higher-level goals and constraints expressed as temporal logic specifications, e.g., on the order in which goals should be reached and on regions that should or should not be visited.} More specifically, in symbolic robot motion planning, the workspace is discretized, and different regions are given ``symbolic'' labels. A set of specifications for the robots can then be given in terms of these symbolic labels, e.g., ``go to locations A and B while avoiding obstacles.'' Plans that meet the specifications are generated based on the discretized representation of the workspace. These high-level discrete plans are then converted into low-level reference trajectories and hybrid control laws based on robot dynamics in the continuous workspace such that the specifications met in the discrete representation are preserved~\cite{belta2005discrete}.  {This preservation of properties between the continuous and discrete workspace representations is important, since a major goal of symbolic motion planning is providing strong guarantees on the correctness of generated plans. In summary, symbolic motion planning approaches seek to guarantee properties at the discrete, symbolic level while also satisfying the types of continuous-level properties that would be addressed by traditional motion planning approaches.}

{Research in symbolic motion planning has focused on a number of areas, including efficient workspace discretization and plan computation \cite{fainekos2005temporal}, computation of control protocols for executing plans guaranteed to meet specifications in the presence of uncontrolled factors in the environment \cite{wongpiromsarn2012receding}, sampling-based approaches \cite{karaman2009sampling} and approximate dynamic programming approaches \cite{papusha2016automata} that do not require explicit discretization of the workspace. However,} challenges remain in addressing the scalability of these approaches for multi-robot systems~\cite{kloetzer2010multirobot,kloetzer2010icra} and in incorporating analysis of human behaviors to improve joint human-robot system performance when a human is able to interact with the system during plan execution~\cite{feng2015controller,fu2015pareto}. In particular, the specification language has been restricted to static and a priori known environments, and the extension to multi-robot systems is challenging due to the infamous ``state-space explosion" problem since both the abstraction and the synthesis algorithms scale exponentially with the dimension of the configuration space~\cite{ClGrPe-MIT-00,saha2014automated}. Since the explored system state needs to be stored in memory, this makes the problem intractable for systems with realistic size.  Hence, there has been limited work on symbolic motion planning for fully autonomous multi-agent systems~\cite{KaLi-Automatica-11,saha2014automated,aksaray2015distributed,tumova2016multi} and very few recent results on HRI systems based on formal verification~\cite{feng2015controller,li2014synthesis,fu2015pareto}. 
{
The papers~\cite{tumova2016multi,KaLi-Automatica-11} investigate the task decomposition problem to address the scalability in multi-agent planning. 
The work \cite{tumova2016multi} converts the linear temporal logic (LTL) specification to B{\"u}chi automaton and uses automaton-based model checking for multi-agent planning. To find the interdependency among agents, an event-based synchronization approach is proposed.
The work~\cite{KaLi-Automatica-11} also uses automaton-based model checking for decentralized control of multi-agent systems. The decomposability and framework of parallel decomposition are discussed. The paper~\cite{saha2014automated} develops an offline compositional multi-robot motion planning framework based on precomputed motion primitives and employs a satisfiability modulo theories (SMT) solver to synthesize trajectories for the individual robots. The work~\cite{aksaray2015distributed} considers a multi-agent persistent surveillance problem 
by minimizing the sum of time between two consecutive visits to regions of interest while satisfying their LTL specifications. A receding horizon controller with automata-theoretic approach is designed to plan each agent's trajectory independently using only local information. The work~\cite{alur2016compositional} proposes a framework for controller synthesis based on compositional reactive synthesis for multi-agent systems. Different objectives such as collision avoidance, formation control and bounded reachability are considered.} 
{We note that our work differs from these previous works in several ways. With regard to work on multi-agent systems, incorporating uncertainty due to both the human and a priori unknown obstacles in the environment creates new challenges. For instance, \cite{tumova2016multi} provides a distributed approach to multi-robot planning that is similar to ours, but it does not account for possible collisions between robots or uncertainty due to human interaction, which must be handled carefully in order to avoid problems such as deadlock and livelock. With regard to work on HRI systems, our approach does not rely on traditional automata-based or similar models of human behavior in order to make guarantees on system performance, allowing us to address more dynamic characteristics such as trust.} We also note that there is a vast body of literature on robot motion planning, and the presented approach serves as a step towards high-level motion planning for robotic systems with a human-in-the-loop at the discrete, symbolic level rather than an alternative to existing motion planning approaches at the continuous level. 

In this paper, we investigate methods for improving the scalability, safety, performance, and adaptability of symbolic motion planning for multi-robot systems, taking into account the effects of human trust. Extending our preliminary works~\cite{spencer2016iros,MaWa-DSCC-2016}, the main contributions of this paper are as follows. 
\begin{enumerate}
	\item We develop a dynamic, quantitative, and probabilistic human-to-robot trust model to compute the evolution of trust during real-time robotic operations. Our proposed trust model integrates the time-series trust model~\cite{sadrspringer2015,WaShZhWa-CPS-15,SaSaWa-CASE-2016} and the Online Probabilistic Trust Inference Model (OPTIMo)~\cite{xu2015optimo} to compute trust quantitatively - a critical first step to enable trust-based multi-robot symbolic motion planning;
	\item We propose trust-based specification decomposition to address the scalability issue in distributed multi-robot symbolic motion planning. Compositional reasoning approaches are used to decompose the global task specification into local specifications for each robot. Atomic propositions (AP) that relate to observation, communication, and control~\cite{filippidis2012cdc} 
	are developed to resolve collisions between robots;
	\item Since high-level human planning is beneficial in certain circumstances but difficult to verify, we design a trust-triggered real-time switching framework to automatically switch between manual and autonomous robot motion planning for tradeoffs between task safety and efficiency. {Our main idea is to switch from autonomous to manual mode once sufficient trust has been developed and utilize human judgment to underpin riskier but more efficient solutions to path planning. This is the opposite to a number of other conceptions of HRI where the switch from autonomous to manual mode results from a loss of trust in the robot (e.g. the inverse trust metric in~\cite{floyd2015improving} that a robot uses to adapt behavior to increase operator's trust).} Also note that here we focus on the high-level switching strategies; the low-level continuous robot motion execution under either manual or autonomous motion planning is still automatic. This differs from the literature on automaton-based switching/shared control synthesis for human-robot teams~\cite{li2014synthesis,fu2015pareto};
	\item Deadlock- and livelock-free
algorithms are proposed to guarantee reachability of all the goals with a human-in-the-loop. We provide a formal proof for the correctness of the overall algorithm and motivate its need through a
discussion of inter-robot collision scenarios that cannot be resolved by purely using the methods summarized in (2);
	\item We perform non-trivial multi-robot simulations with direct human inputs (e.g., through gamepad and mouse with GUI designs) to demonstrate the real-time trust computation and the implementation of the proposed trust-based symbolic robot motion planning methods. We explore these methods in the context of an intelligence, surveillance, and reconnaissance (ISR) scenario. Other possible applications include search and rescue, DARPA Urban Challenge, warehouse management, intelligent transportation systems, etc. 
\end{enumerate}


The remainder of this paper is organized as follows. Section \ref{sec:humanRobotInteraction} introduces the problem setup and symbolic robot motion planning. Section \ref{sec:trustModel} presents a dynamic, quantitative, and probabilistic model of human-to-robot trust. Section \ref{sec:trustBasedSpecificationDecomposition} outlines the method for trust-based specification decomposition, and Section \ref{sec:realTimeTrustBasedSwitching} describes the method for switching between manual and autonomous motion planning. Section \ref{sec:liveness} proposes deadlock- and livelock-free algorithms to guarantee goal reachability with a human-in-the-loop. A set of simulations integrating Matlab and the NuSMV model checker~\cite{NuSMVlink} with direct human inputs is presented in Section \ref{sec:simulation}. Concluding remarks are given in Section \ref{sec:conclusion}.

\section{Symbolic Robot Motion Planning Considering HRI}
\label{sec:humanRobotInteraction}

\subsection{Problem Setup}
We consider an ISR scenario in which a team of $N$ robots represented by an index set $I_R = \{1, \ldots, N\} \subset \mathbb{Z}$, supervised by a human operator, must reach a set of $M$ goal destinations represented by the set $Goals = \{1, \ldots, M\} \subset \mathbb{Z}$ while avoiding collisions with stationary obstacles and with each other (i.e. mobile obstacles), as shown in Fig. \ref{fig:app1}. 
As is standard in symbolic robot motion planning, the workspace is discretized into  $Q$ regions or states represented by the set $W=\{w_1,w_2,\cdots,w_Q\}$, which are labeled with relevant properties (e.g., whether they contain an obstacle or goal). The regions can be in different shapes such as points, triangles, polytopes, or rectangles and the union of regions can be used to approximate the workspace of robots in an arbitrarily close manner. Depending on the robot dynamics and the corresponding low-level control, different cell decomposition strategies can be applied, as discussed later. Note this discretization can be performed to an arbitrary degree of accuracy; however, increasing the number of regions significantly increases the computational complexity of the planning problem. Most discretizations therefore significantly overapproximate certain features of the workspace, e.g., only a small portion of a region labeled as ``obstacle'' contains a real obstacle. The consequence is that planning through the workspace may be overly conservative, since paths that go through regions containing obstacles might be feasible in the continuous workspace. Though theoretically a finer discretization of the space would allow for computation of these paths, this may not be desirable in practice since it will lead to a very large state space which requires significant computational resources. 
In such a case, a combination of human judgment, assistance, and permission might be needed in order to make the attempt, motivating the need for a human operator.

\begin{figure}	
\centering
    \includegraphics[width=4.5in]{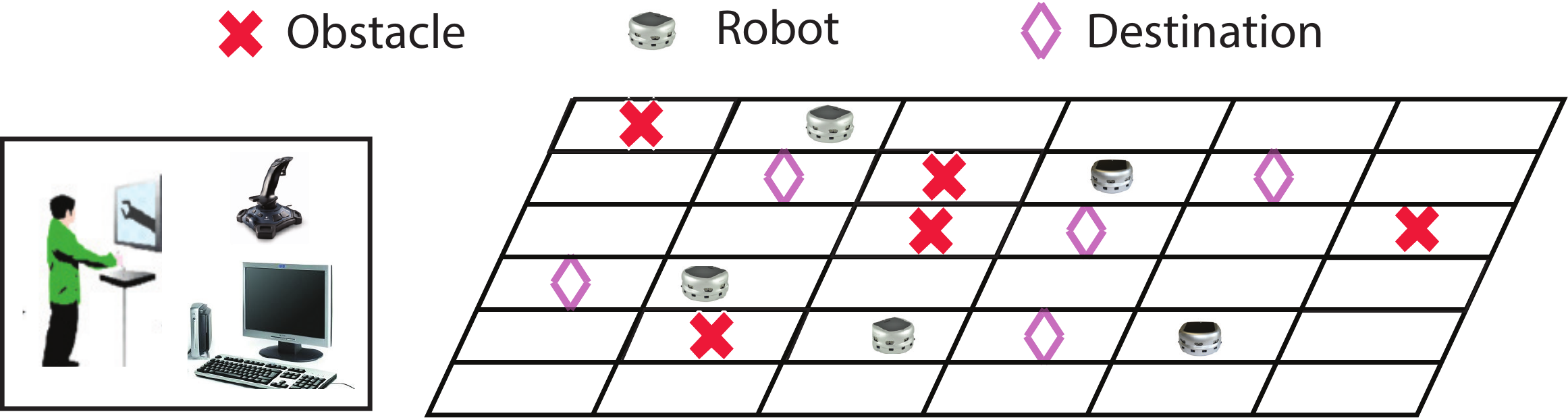}
	\caption{Multiple robots must reach a set of destinations while avoiding obstacles and collisions with other robots, taking shorter but riskier paths between obstacles with human oversight when trusted to do so. That is, when the human operator trusts a robot, he/she plans for robot motion through visual feedback from a robot's onboard camera using HMI such as keyboard, mouse, or joystick.}
	\label{fig:app1}
\end{figure}

We assume that the set of $M$ goal destinations is known from the start, and each goal must be reached by at least one robot while collisions with obstacles and between robots are avoided. This set of requirements forms a specification for the scenario. Our proposed planning scheme is implemented in a distributed manner, making use of compositional reasoning approaches to decompose the global specification. More specifically, the goal portion of the specification is decomposed such that each robot is assigned a subset of the goal destinations and locally synthesizes a plan to reach them. For obstacle avoidance, we assume obstacle locations are not known a priori, and so when a robot discovers an obstacle, it re-synthesizes a plan to reach its remaining goals after updating its representation of the workspace. We show that our planning approach is guaranteed to meet these specifications under some mild assumptions. We then take this planning approach and incorporate a method for inter-robot collision avoidance in which the involved robots locally collaborate to re-synthesize plans. We later modify the algorithm to guarantee the reachability of all $M$ goal destinations by
eliminating the possibility of deadlock and livelock. 

With respect to human interaction, a quantitative, probabilistic, and dynamic trust model based on robot performance, human performance, joint human-robot fault, human intervention, and feedback evaluation is used to estimate human trust in each of the robots throughout the scenario.
This estimate of trust affects the specification decomposition, with more trusted robots assigned more destinations. Trust is also used to determine when the robot should suggest navigating between obstacles, as this requires real-time switching between manual and autonomous motion planning.
Human consent for this switching is assumed to depend on the change of trust as well as whether or not the human is currently occupied with other tasks.
Each robot is assumed to follow a simple first-order kinematic equation of motion 
\begin{equation}
\dot{x}_i=u_i,~i\in I_R=\{1,2,\cdots,N\},
\label{eq:dynamics}
\end{equation}
where $x_i$ is the robot position and $u_i$ is the motion control input. We assume homogeneous robots with the same dynamics. 
To guarantee that the high-level discrete motion plan can be realized by some low-level continuous motion control, we use a rectangular partition and a linear quadratic regulator (LQR) controller~\cite{kirk2012optimal} for the robot kinematics (\ref{eq:dynamics}) so that the robot always seeks to reach the centroid of the next cell in the discrete path, i.e.,
\begin{equation}\label{eq:AP_control}
u_i(t)=-K_i(x_i(t)-x_i^d),~K_i=R^{-1}_iP_i(t),
\end{equation} 
where $K_i$ is some control gain, $x_i^d$ represents the centroid of a discrete cell, $P_i(t)$ is found by solving the Riccati equation $P_i(t)R^{-1}_iP_i(t)=Q_i$, and $R_i$, $Q_i$ are the weights on state and control input with a quadratic cost, respectively. Therefore, under this setup, for the high-level discrete plan, there always exists a low-level continuous control under (\ref{eq:dynamics}) given the rectangular cell partition. This guarantees the robot will never enter an unplanned region while moving between sequential regions in the planned path, allowing us to establish a bisimulation relation between the continuous state space used for control of the robot and the discretized state space used for planning~\cite{alur2000discrete}. 
Since we assume homogeneous multi-robot systems, we can use the same the partition
of the workspace for each robot for autonomous motion planning.
In more general cases, it has been proven that for any robot with affine dynamics (e.g. the kinematics (\ref{eq:dynamics}) is a special case of affine dynamics), the low-level continuous control and high-level discrete plan are bisimular under polygon triangulation~\cite{belta2005discrete}. Open source toolboxes such as TuLip (\url{http://tulip-control.sourceforge.net}) can be used to perform the partitioning/triangulization. 
Under manual motion planning, when the human trusts a robot, he/she will assign waypoints as high-level guidance. In that case, rather than moving between cell centroids, the robot will navigate autonomously between successive waypoints using the low-level controller (\ref{eq:AP_control}) where $x_i^d$ becomes the position of a waypoint, and no partition is needed. {In our system, the human can see the grid map as displayed in Fig.~\ref{fig:manualpath} and Fig.~\ref{fig:manual_path_final}, trust level (Fig.~\ref{fig:GUI}(a)), {visual information,} and support robot requests via {the robot's onboard camera and} HMIs (Fig.~\ref{fig:humaninput}) and/or graphical user interface (GUI) (Fig. \ref{fig:GUI}(b) and (c)) as depicted in Fig. \ref{fig:app1}.} More details regarding the HMI and GUI designs will be provided in the simulation section \ref{sec:simulation}.

Fig. \ref{fig:teamSchematic} shows the schematic of the proposed trust-based human-robot collaboration system and details the specification decomposition process and switching framework for one robot. 
Because robots are good at local tasks given their limited computing, sensing, and communication capabilities, an in-situ autonomous motion planner is developed based on local sensing and communication information to guarantee safe completion of the task. However, since humans excel at high-level planning, we allow the human to intervene in the robot motion planning if required to increase task efficiency. 
This combination of autonomous and manual motion planning allows the joint system to achieve
the global specification efficiently but without overloading the human operator. As planning and execution proceeds,
human trust in each robot evolves dynamically, with the human
choosing to collaborate more often with trusted robots.
To reduce computational complexity, the global specification for the HRI team is decomposed using compositional reasoning methods, with more trusted robots assigned more tasks and vice versa. To further integrate human intelligence while guaranteeing task safety, a trust-based real-time switching framework is developed, and an autonomous decision-making aid is designed. By default, the autonomous and safe motion planning is implemented to guarantee task completion based on an over-approximation of the task domain, ensuring that obstacles are avoided. Each robot then carries out the control synthesis process (including local configuration space abstraction, local path planning, and execution) based on its decomposed local specification. 
However, if the human trusts a robot and is not overloaded, manual motion planning may be requested for more efficient but riskier solutions, e.g., moving between close
obstacles. In such a case, the robot behavior under manual motion planning needs to be monitored so that if
any event that will violate the local task specification is detected, the autonomous motion planning will be activated again. {At this moment, the robot can either alert the operator to this change of operation mode, or leave it to the operator to see that an obstacle has been detected and notice that the robot has returned to its autonomous mode of operation\footnote{{In our simulation Section~\ref{sec:simulation}, the robot does not alert the change of mode when the autonomous mode is reactivated.}}.} Robot and human performance measurements, joint human-robot fault measurements, direct human intervention, and trust evaluation are used as feedback to update the trust model and specification decomposition. 

\begin{figure}[t]
\centering
\includegraphics[width = 4.0in]{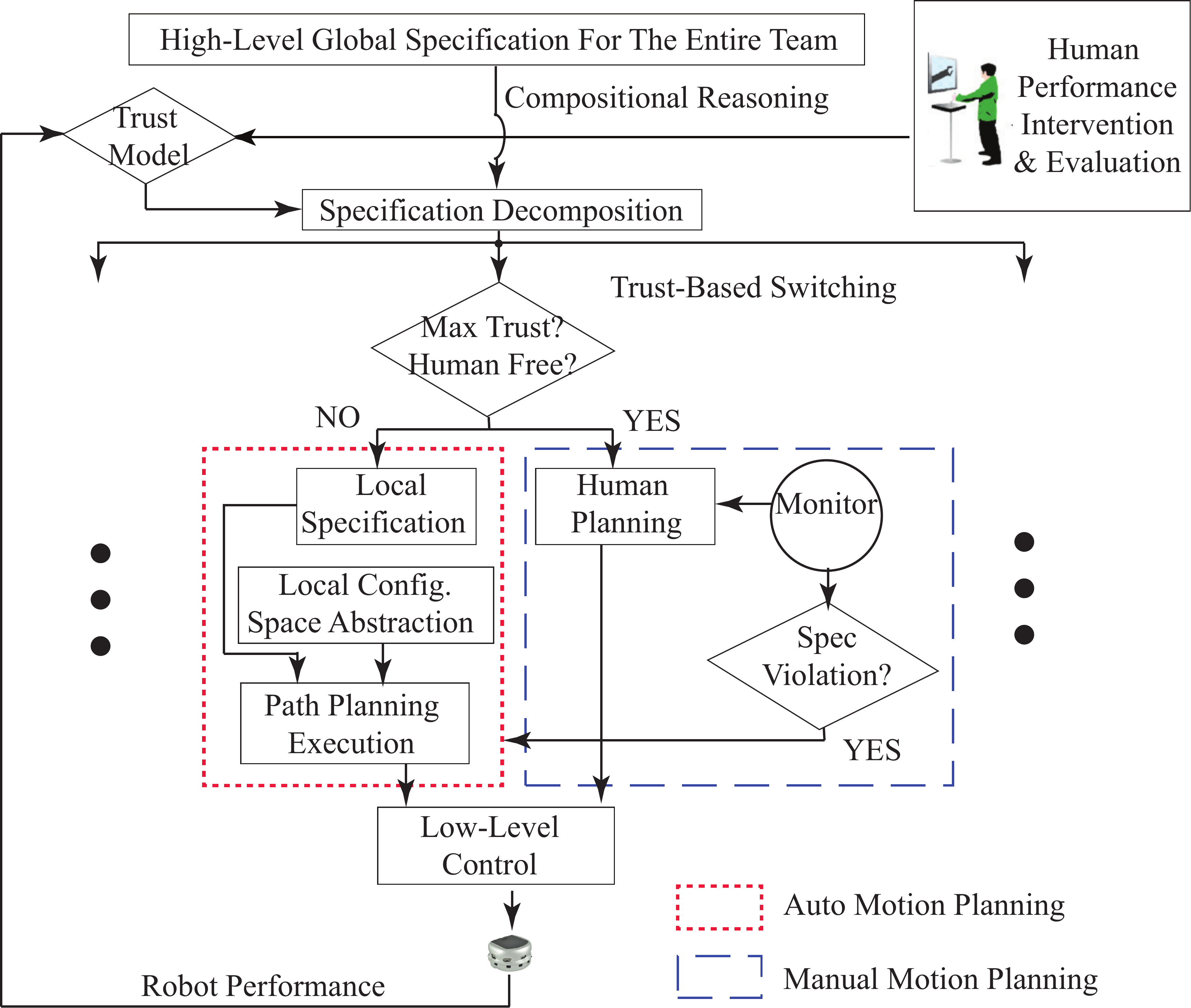}
\caption{Schematic diagram of trust-based symbolic motion planning for human and multi-robot collaboration systems.}
\label{fig:teamSchematic}
\end{figure}

\subsection{Symbolic Robot Motion Planning}
\label{subsec:symbolicMotionPlanning}


Symbolic robot motion planning is performed over a discrete abstraction of the robot's workspace, as previously
described. The set of discrete cells, the properties each cell satisfies, and the movements allowed between
cells in the workspace are often represented as a finite transition system, defined as follows. 

\begin{definition}[Finite Labeled Transition System]
	\label{def:TS}
	A finite labeled transition system is a tuple \mbox{$TS= (S, Act, \delta, I, \Pi, L)$} consisting of 
	\begin{enumerate}
		\item a finite set of states $S$,
		\item a finite set of actions $Act$,
		\item a transition relation $\delta \subseteq S \times Act \times S$,
		\item a set of initial states $I \in S$,
		\item a set of atomic propositions (AP) $\Pi$,
		\item and a labeling function $L : S \rightarrow 2^{\Pi}$.
	\end{enumerate}
	We define a path of a transition system as a sequence of states $\sigma = s_0 \overset{\alpha_0}{\rightarrow} s_1 \overset{\alpha_1}{\rightarrow} s_2 \overset{\alpha_2}{\rightarrow} \ldots$, where $s_k \in S$ is the state at step $k \geq 0$, $s_0 \in I$, and $(s_k, \alpha_k, s_{k+1}) \in \delta$ with $\alpha_k\in Act$. We define a trace as a sequence of labels that denote which APs are true in each state, i.e.,  $L(\sigma) = L(s_0) L(s_1) L(s_2) \ldots$.  
\end{definition}

Let us define a transition system $TS_w$ to represent the discrete abstraction of the workspace. Then $W$ is the set of cells in the workspace, $Act_w$ is the set of actions that represent the robot moving between cells, $\delta_w$ encodes pairs of adjacent cells, $I_w$ is the set of possible starting cells (in this case, all cells without obstacles), $\Pi_w$ contains atomic propositions $\pi^g_j for j \in Goals$ and $\pi^b_j for j \in Obs$ representing the presence of obstacles or goals in workspace cell $j$, and $L(\cdot)$ labels cells that contain obstacles or goals. At the start of the scenario, each robot $i\in I_R$ performs planning using a local copy $TS_i$ of $TS_w$ in which obstacle locations are not initially known, i.e., the labeling function $L_i(\cdot)$ does not initially return $\pi^b_j$ for any $s \in S_i$. As each robot $i$ learns the location of obstacles, the labeling function of $TS_i$ is updated.


Specifications for symbolic robot motion planning can be expressed as LTL formulae \cite{baier2008principles}. An LTL formula $\varphi$ is formed from APs, propositional logic operators, and temporal operators according to the grammar
\begin{equation}
\varphi::=\textrm{true} \mid \pi \mid \lnot\varphi \mid \varphi_1\land\varphi_2 \mid \bigcirc\varphi \mid \varphi_1\mathbin{U}\varphi_2,
\end{equation}
where $\pi$ is an AP, $\land$ is the propositional logic operator ``and,'' $\bigcirc$ is the temporal operator ``next,'' and $U$ is the temporal operator ``until.'' From these, the standard propositional operators $\lor$ ``or,'' $\rightarrow$ ``implies,'' etc. can be derived in addition to the temporal operators $\square$ ``always" and $\lozenge$ ``eventually." The formula $\square\pi$ is true for a trace if propositional formula $\pi$ is true in every state of the trace, $\lozenge\pi$ is true if $\pi$ is true in some state of the trace, and $\lozenge\square\pi$ is true if $\pi$ is true in some state of the trace and all states thereafter.
LTL can be used to specify expressive motion tasks such as reachability, (``reach $\pi$ eventually", i.e., $\lozenge\pi$), obstacle avoidance and safety (``always avoid $\pi$", i.e., $\square\neg\pi$), convergence and stability (``reach $\pi$ eventually and stay there for all future times", i.e., $\lozenge\square\pi$), and sequencing and temporal ordering of different tasks that can otherwise be difficult to specify in traditional path planning.


Given an LTL plan specification $\varphi$ and a finite labeled transition system $TS$, a high-level discrete plan that satisfies the specification can then be synthesized using a model checking approach. Traditionally, a model checker verifies whether a system $TS$ satisfies a specification $\varphi$, written $TS \models \varphi$. If not, it returns a counterexample trace $\sigma \in \textrm{traces}(TS)$ that does not satisfy the specification, i.e., $L(\sigma) \nvDash \varphi$. Model checking $TS$ against the negation of the specifications returns a trace that satisfies the specifications, since $L(\sigma) \nvDash \lnot \varphi \rightarrow L(\sigma) \models \varphi$. For symbolic robot motion planning,  this approach generally produces short paths relatively quickly \cite{humphrey2014formal}. Here, we generate plans using this approach with the NuSMV model checker~\cite{NuSMVlink}. For a team of robots $TS_i,~i=\{1,2,\cdots,N\}$ with each $TS_i$ as the discretized abstraction of the robot motion within the workspace, to satisfy a feasible global specification $\varphi$, we can then use this approach with the model checker NuSMV to generate motion plans for the multi-robot team. 
That is, the motion plan for each robot will
encode a single path $\sigma^i =s_0^i \overset{\alpha_0^i}{\rightarrow} s_1^i\overset{\alpha_1^i}{\rightarrow}s_2^i\ldots $ obtained by model checking, where $s_k^i \in S_i$ is the robot $i$'s state at step $k = 0,1,\ldots$ and pairs of sequential states $s_k^i\overset{\alpha_k^i}{\rightarrow} s_{k+1}^i\in \delta_i$ with $\alpha_k^i\in Act_i$ represent feasible transitions between states, i.e., direct transitions between states that are achievable in the continuous workspace. 

For the ISR scenario illustrated in Fig. \ref{fig:app1}, we are interested in specifications of the form
\begin{equation}
	\varphi=\underbrace{\bigwedge_{j\ \in Goals}\lozenge\pi^g_j}_{\text{Reachability}}\land \underbrace{\bigwedge_{i=1} ^{N}\bigwedge_{j \in Obs}\Box(\pi_{ij}^o\to\lnot\bigcirc\pi^b_j)}_\text{Obstacle~Avoidance}\wedge \underbrace{\bigwedge_{i=1}^{N}\Box(\pi_i^c\to\lnot\bigcirc\pi_i^u)}_\text{Robot~Collision~Avoidance}
	\label{eq:globalSpec}.
\end{equation}

The above LTL specification consists of a set  $\Pi$ of APs whose truth values are determined by the workspace that the human-robot team is operating in and observed by the robots' sensors. The value of each $\pi\in\Pi$ is a Boolean valuable.
The propositions of the form $\pi^g_j$ for $j \in Goals$ and $\pi^b_j$ for $j \in Obs$ label regions containing \textit{Goals} and \textit{Obstacles}. The term $\bigwedge_{j\ \in Goals} \lozenge \pi_j^g$ indicates the reachability specifications, i.e., all the goals will be eventually visited by a robot. Propositions of the form $\pi_{ij}^o$, $\pi_i^c$, and $\pi_i^u$ represent the observation, communication, and control APs. The term $\bigwedge_{i=1} ^{N}\bigwedge_{j \in Obs}\Box(\pi_{ij}^o\to \lnot \bigcirc\pi_j^b)$ indicates the obstacle avoidance specification, i.e., it is always true that once robot $i$'s sensor observes an obstacle $j$, the robot will not move to cell $w_j$ labeled with $\pi_j^b$ at the next step. The term $\bigwedge_{i=1}^{N}\Box(\pi_i^c\to \lnot\bigcirc \pi_i^u)$ indicates the inter-robot collision avoidance specification, i.e., it is always true that once robot $i$ finds another robot within its communication range, they will exchange goal, obstacle, and planned path information, and robot $i$ will not implement the LQR controller (\ref{eq:AP_control}) at the next step (but wait or replan). 
Note that the while the set $Goals$ is known by each robot at the start of the scenario, the set $Obs$ is not. Rather, each robot must update its transition system model over time as it discovers obstacles, replanning when necessary. Similarly, propositions of the form $\pi_{ij}^o$, $\pi_i^c$, and $\pi_i^u$ can only be evaluated during execution and are therefore only used to trigger the replanning process. This allows us to focus on distributed multi-robot symbolic motion planning for scalability. More details regarding the specification decomposition will be further explained in Section \ref{sec:trustBasedSpecificationDecomposition}.  

\section{Computational Trust Model}
\label{sec:trustModel}

Existing trust models include argument-based probabilistic models~\cite{cohen1998trust,moray1999laboratory}, qualitative models~\cite{moray1999laboratory,jian2000foundations}, time-series models~\cite{lee1992trust,moray2000adaptive,lee2004trust,gao2006extending}, neural net models~\cite{farrell2000connectionist}, regression models~\cite{de2003effects}, and the most recent Bayesian dynamic network based model called OPTIMo~\cite{xu2015optimo}. In particular, the time-series model characterizes the dynamic relationship between human trust in robots and the independent variables. Unlike most other types of models, the time-series model proposed by Lee and Morray is based on prior trust, robot performance, and joint human-robot fault~\cite{lee1992trust}. This model reflects the dynamic changes of trust and may be suitable for both real-time analysis and prediction of trust for control allocation. The OPTIMo trust model uses a dynamic Bayesian network (DBN) to infer beliefs over the human's latent trust states, or the degree of trust in near real-time, based on the history of observed interaction experiences. Although a performance-centric model, OPTIMo offers a probabilistic and dynamic framework to estimate and predict trust in complex tasks. There are also several computational trust models available~\cite{hall1996trusting,mikulski2013phd}. However, these models only characterize a robot's trustworthiness based on its task performance and do not consider the human influence. In this section, we integrate the time-series trust model and the OPTIMo trust model and develop a new mathematical model for estimating human trust in a robot, which is dynamic, quantitative, and probabilistic, and takes into account joint human-robot performance.

Since human behaviors are notoriously difficult to model, predict, and verify and the task environment is usually complex and uncertain, probabilistic analysis must be utilized to capture these uncertainties in trust estimates.
Our proposed integrated trust model is shown in Fig. \ref{fig:hmm}. First, let $T_i(k)$ denote human trust in robot $i,~i=1,2,\cdots,N$, which is a hidden random variable taking values from 0 to 1. This assumption is made because trust is difficult to measure directly in real-time and is usually measured subjectively after each experiment session. We use solid green ellipses to represent $T_i$ in the figure, indicating that this is a process evolving in real-time. The discrete trust state results in a standard DBN model. The actual realization and the sequence of the process, i.e., dynamic evolution of trust over time, is hidden. Based on our previous works involving creation of a time-series trust model~\cite{WaShZhWa-CPS-15,sadrspringer2015,SaSaWa-CASE-2016}, we identify three major factors impacting trust, i.e., robot performance $P_{R,i}$, human performance \YueE{$P_{H,i}$}, and joint human-robot system fault $F_i$. These factors are shown in solid yellow ellipses in the figure. Following the OPTIMo model, we also have the human inputs $m_i(k),c_i(k),f_i(k)$ represented by solid and dashed blue ellipses in the figure, with dashed ellipses indicating intermittent observations. This is because it might not be practical to have human inputs all the time. The term $m_i(k)\in\{0,1\}$ represents human intervention (i.e., switches between manual and autonomous modes) in motion planning, and its default value of zero indicates no  intervention. Hence, $m_i(k)$ can be measured and updated in real-time. The term $c_i\in\{-1,0,+1\}$ represents change in trust as reported by the human, with -1 indicating a decrease in trust, 0 indicating no change, and +1 indicating an increase in trust. The term $f_i\in(0,1)$ represents subjective trust feedback, which is a continuous value between 0 and 1. Both $c_i$ and $f_i$ only require occasional observations. That is, the participants will only be asked to provide trust change $c_i(k)$ and trust feedback $f_i(k)$ periodically. This ensures there is not much additional cognitive workload for the human operator during multi-robot cooperative tasks.  

\begin{figure}
	\centering
	\includegraphics[width=3.5in]{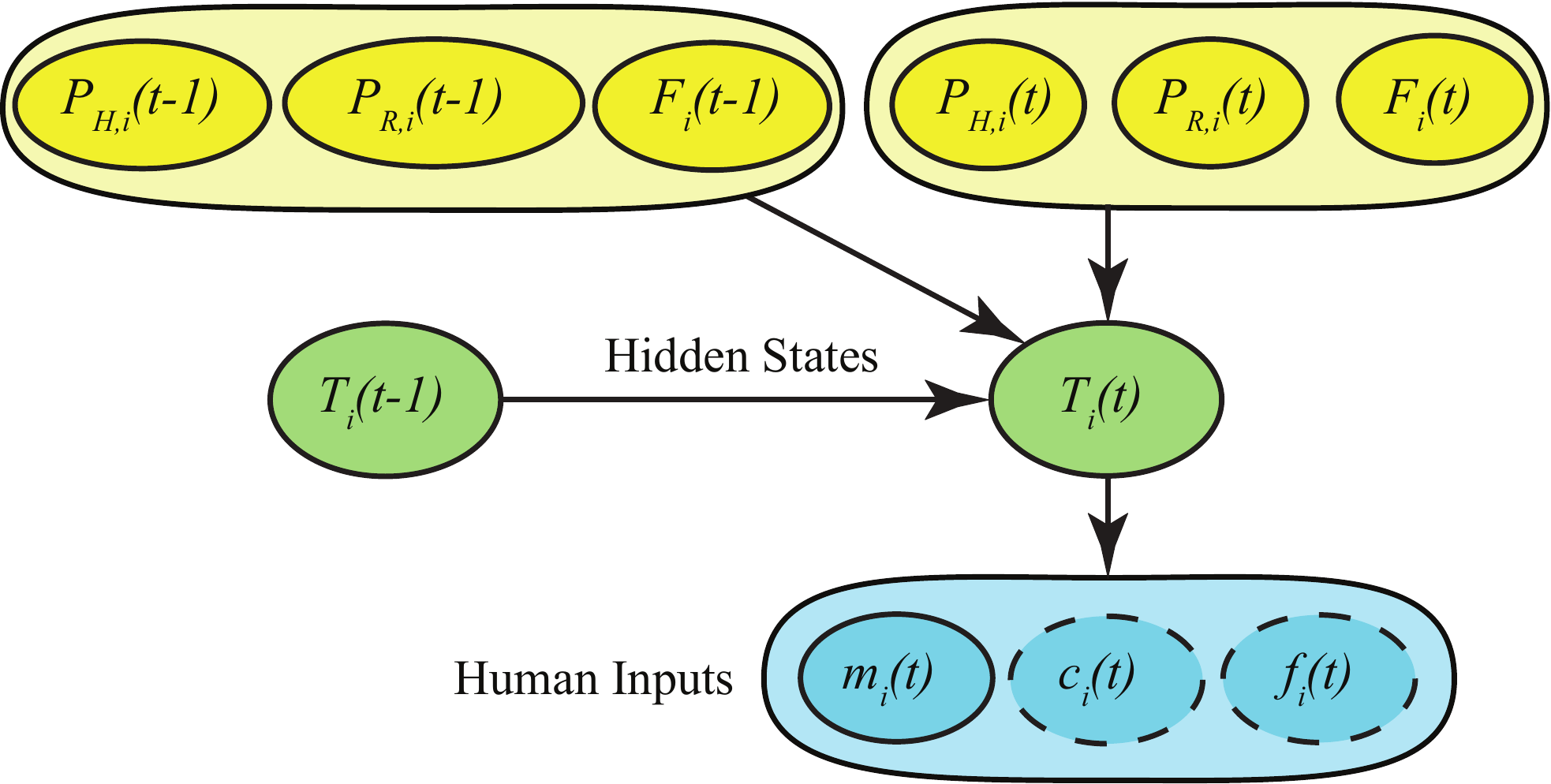}
	\caption{\small{A dynamic Bayesian network (DBN) based model for dynamic, quantitative, and probabilistic trust estimates.}}\label{fig:hmm}
\end{figure}

The conditional probability distribution (CPD) of human trust in robot $i$ at time $t$ based on the previous trust value $T_i(t-1)$ and the above causal factors can be expressed as a Guassian distribution with mean value $\bar{T}_i(t)$ and covariance $\sigma_i(t)$:
\begin{eqnarray}
\label{eq:guass_cpd}
&&\textrm{Prob}(T_i(t)|T_i(t-1), P_{R,i}(t), P_{R,i}(t-1),\YueE{P_{H,i}(t)},\YueE{P_{H,i}(t-1)}, F_i(t),F_i(t-1))\nn\\
&=&\mathcal{N}(T_i(t);\bar{T}_i(t),\sigma_i(t)),
\end{eqnarray}
\begin{eqnarray}
&&\hspace{-0.2in}\textrm{where~~}\bar{T}_i(t)=AT_i(t-1)+B_1P_{R,i}(t)-B_2P_{R,i}(t-1)+C_1\YueE{P_{H,i}(t)}-C_2\YueE{P_{H,i}(t-1)}\nn\\
\label{eq:trust}
&&\hspace{0.55in}+D_1F_i(t)-D_2F_i(t-1).
\end{eqnarray}
Here, $\bar{T}_i(t)\in(0,1)$ represents the mean value of human trust in a robot $i$ at time $t$, $P_{R,i}\in(0,1)$ represents performance of robot $i$, \YueE{$P_{H,i}\in(0,1)$} represents human performance, $F_i\in(0,1)$ represents faults made in the joint human-robot system and $\sigma_i$ reflects the variance in each individual's trust update. The coefficients $A,B_1,B_2,C_1,C_2, D_1, D_2$ can be determined by data collected from human subject tests (see the author's previous work~\cite{sadrspringer2015} for more details). Here, these coefficients are further scaled such that $\bar{T}_i$ is normalized. See Fig.~\ref{fig:fTrust} for an illustration of the dynamics of mean trust $\bar{T}_i$.

In this scenario, robot performance \YueE{$P_{R,i}$} is modeled as a function of ``rewards'' the robot receives when it identifies an obstacle or reaches a goal destination:
\begin{equation}
\YueE{P_{R,i}(t)} = C_O \frac{N_{O_i}(t)}{\sum_{i=1}^N N_{O_i}(t)} + C_G \frac{N_{G_i}(t)}{\sum_{i=1}^N N_{G_i}(t)}
\nonumber\end{equation}
where $N_{O_i}$ and $N_{G_i}$ are the number of obstacles detected and goals reached by the robot $i$ up to time $t$, and $C_O\in(0,1)$ and $C_G=1-C_O\in(0,1)$ are corresponding positive rewards chosen such that \YueE{$P_{R,i}(k)$} is normalized. This allows the robot to earn trust as it learns details of the environment. 

Human performance is calculated based on workload and the complexity of the environment surrounding the robot with which the human is currently collaborating. The concept of utilization ratio, $\gamma$, is used to measure workload~\cite{spencer2015slqr}
\begin{eqnarray}
\gamma(t) &=& \gamma(t-1)+\frac{\sum_{i=1}^{N}m_i(t)-\gamma(t-1)}{\tau} 
\label{eq:utilizationRatio}
\end{eqnarray}
where $m_i(t) = 1$ if the human is collaborating with robot $i$ and $0$ otherwise, and $\tau$ can be thought of as the sensitivity of the operator.
Assuming a human can only collaborate with one robot at a time, i.e., manually assigning paths through obstacles for the chosen robot, (\ref{eq:utilizationRatio}) allows workload to grow or decay between 0 and 1. Complexity of the environment is based on the number of obstacles that lie within sensing range $r_i$ of collaborating robot $i$ at time $t$. 
The human's superior capability in creating more detailed paths will be enhanced in more complex environments, leading to increased performance in the presence of more obstacles. On the other hand, human performance decreases with respect to workload. Therefore, \YueE{$P_{H,i}(t)$} can be modeled as follows:
\begin{equation}
\YueE{P_{H,i}(t)}=  \left\{\begin{array}{l l}
1-\gamma(t)^{S_{o_i}(t) + 1} & \textrm{if~}m_i(t)=1 \\
1-\gamma(t) & \textrm{if~}m_i(t)=0
\end{array}\right.
\label{eq:humanPerformance}
\end{equation}
where $S_{o_i}$ is the number of obstacles within sensing range of collaborating robot $i$, reflecting the environmental complexity.
Fig. \ref{fig:PHPlot} shows the change of human performance with respect to workload $\gamma$ 
and environmental complexity $S_{o_i}$.

\begin{figure}
	\centering
	\includegraphics[width = .5\columnwidth]{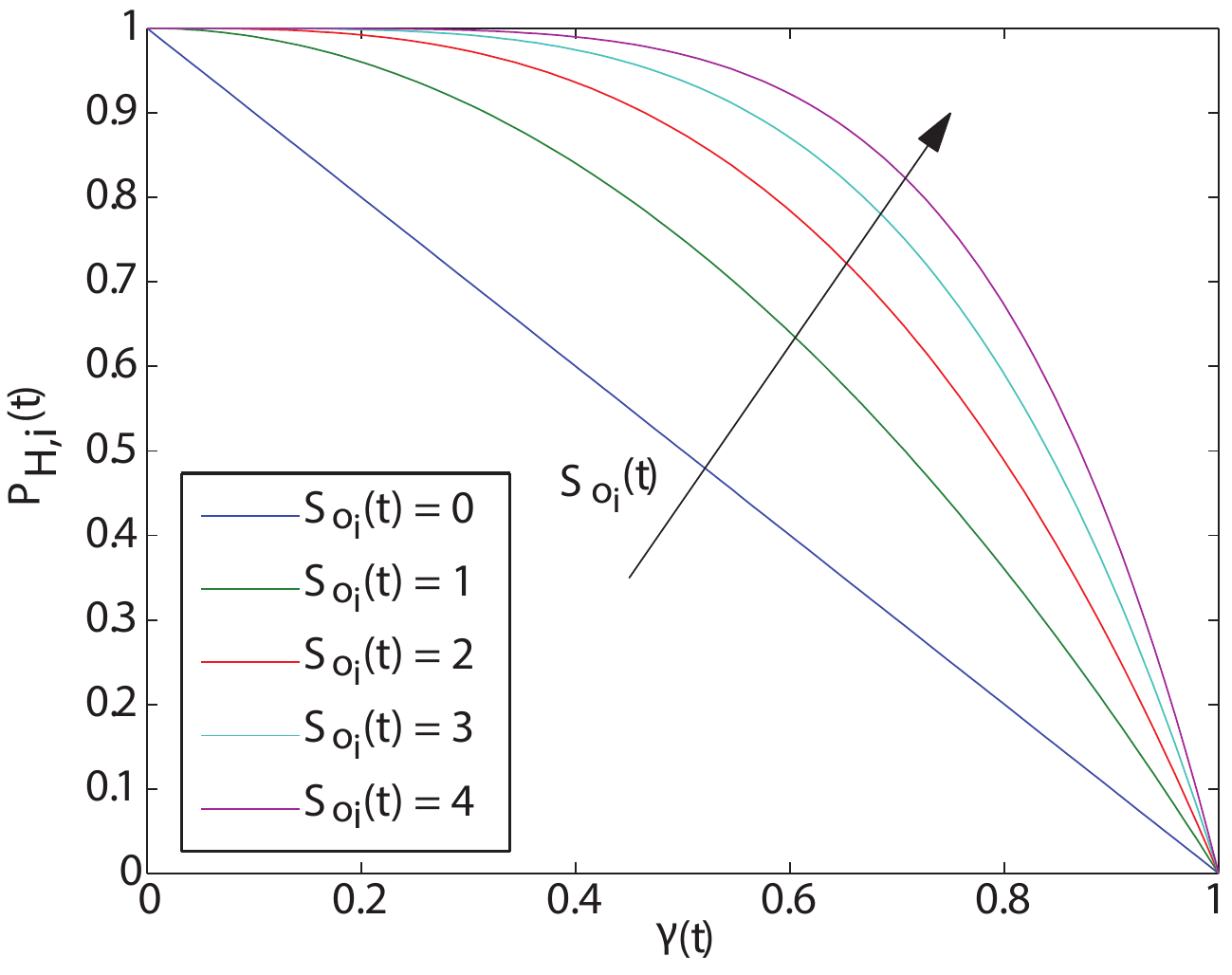}
	\caption{Plot of human performance when collaborating with a robot $i$.}
	\label{fig:PHPlot}
\end{figure}

Faults in the system are modeled as the ``penalty" the robot receives when it enters an obstacle region or detects an obstacle on its planned path: $F_i(t)=-\frac{N_{H_i}(t)}{N_{O_i}(t)}$
where $N_{H_i}(t)$ is the total number of obstacle regions robot $i$ has entered before sensing the corresponding obstacle up to time $t$. Note that faults can originate from both the robot and the human, i.e., human trust in a robot will decrease even if the robot enters an obstacle region under manual motion planning. 

Based on the DBN, we can quickly establish the belief update of trust, i.e., $bel(T_i(t))=\textrm{Prob}(T_i(t)|P_{R,i}(1:t), \YueE{P_{H,i(1:t)}},F_i(1:t),m_i(1:t),c_i(1:t),f_i(1:t),T_i(0))$,
using the forward algorithm by applying the principle of dynamic programming to avoid incurring exponential computation time due to the increase of $t$. We can first compute 
\begin{eqnarray}
&&\overline{bel}(T_i(t),T_i(t-1))\nonumber\\
&=&\textrm{Prob}(m_i(t)|T_i(t), T_i(t-1)) \textrm{Prob}(c_i(t)|T_i(t), T_i(t-1))
\textrm{Prob}(f_i(t)|T_i(t))\cdot\nn\\
&&\textrm{Prob}(T_i(t)|T_i(t-1), P_{R,i}(t), P_{R,i}(t-1),\YueE{P_{H,i}(t),P_{H,i}(t-1),} F_i(t),F_i(t-1))bel(T_i(t-1)),\nonumber
\end{eqnarray} 
where $\textrm{Prob}(m_i(t)|T_i(t), T_i(t-1))$ is the probability of human intervention, $\textrm{Prob}(c_i(t)|T_i(t),T_i(t-1))$ is the probability of a trust change given current and prior trust, and $\textrm{Prob}(f_i(t)|T_i(t))$ is the probability of subjective trust evaulation, respectively, which can follow a similar sigmoid distribution as in~\cite{xu2015optimo}. For example, the CPD of human intervention based on trust can be modeled as follows
\begin{eqnarray}\label{eq:omega}
\textrm{Prob}(m_i(t)=1|T_i(t), T_i(t-1))=
(1+exp(-(\omega_1 T_i(t)-\omega_2 T_i(t-1))))^{-1}
\end{eqnarray}
where $\omega_1$ and $\omega_2$ are positive weights and this CPD indicates higher willingness to collaborate with a robot (intervention in path planning) when the human trust is higher. 
It follows that
\begin{eqnarray}\label{eq:BayesUpdates}
&&
bel(T_i(t))=\frac{\int\overline{bel}(T_i(t),T_i(t-1))\textrm{d} T_i(t-1)}{\int \int \overline{bel}(T_i(t),T_i(t-1))\textrm{d} T_i(t-1)\textrm{d} T_i(t)}.
\end{eqnarray} 
The network parameters for the DBN such as  {$\omega_1$ and $\omega_2$ in Equation (\ref{eq:omega})} can be learned by the well-known expectation maximization (EM) algorithm~\cite{moon1996expectation}. {For the implementation of the EM algorithm, we summarize the steps as follows:
\begin{itemize}
\item Expectation steps: 
\begin{enumerate}
\item Calculate the filtered trust belief $bel(T_i (t))$ at the end of the training simulation $t=t_f$ forward in time assuming a uniform initial trust belief; 
\item Calculate the smooth trust belief
$bel_s (T_i (t))=Prob(T_i (t)|P_{R,i} (1:t_f ),\YueE{P_{H,i}(1:t_f)},F_i (1:t_f ),m_i (1:t_f ),c_i (1:t_f ),f_i (1:t_f ),T_i (0))$
for all data backward in time given that $bel_s (T_i (t_f))=bel(T_i (t_f))$; 
\item Take the expectation for each smooth trust belief to get a single sequence of trust states.
\end{enumerate}
\item Maximization step: Use this calculated sequence of trust states at the Expectation steps, along with other performance $P_{R,i}, \YueE{P_{H,i}}, F_i$ and human inputs $m_i,c_i,f_i$ to find the optimized parameters for each CPD separately.
\end{itemize}
Note that this parameter learning process of the DBN is performed offline during the training session and hence will not affect the functionality of the system and the user experience during the real-time operation.
} A separate trust model should be trained based on each user's experience \YueE{using their respective trust input data $(m_i, c_i, f_i)$ as well as the performance measures $(\YueE{P_{H,i}}, P_{R,i}, F_i)$. }

\section{Trust-Based Specification Decomposition}
\label{sec:trustBasedSpecificationDecomposition}


Available methods for multi-robot symbolic motion planning have mainly focused on fully autonomous systems and can be summarized into two main types: centralized and decentralized solution approaches. Centralized solutions treat the robot team as a whole and have a large global state space formed by taking the product of the state spaces of all the robots~\cite{kloetzer2010icra,kloetzer2010multirobot}, which is too large to handle in practice. Decentralized solutions tend to give local specifications to individual robots, which results in a smaller state space but often sacrifices guarantees on global performance~\cite{filippidis2012cdc}. Here, we propose a distributed solution for human-robot symbolic motion planning. 
Ideally, if all the obstacles are known beforehand and the speed profile of each robot is given, optimal and collision-free paths can be designed offline without the necessity of collaboration. In this work, we consider a gradually learned environment with intermittent communications when robots are within the communication range, which therefore requires collaboration between robots for obstacle and collision avoidance. We first present a method for addressing the collision avoidance task in Section \ref{sec:commObsControl} and then a method for decomposing the specification for individual tasks in Section \ref{sec:compositionalReasoining}. In Section \ref{sec:liveness}, deadlock- and livelock-free
algorithms are developed to guarantee that all goals are reached while all obstacles and robot collisions are
avoided with a human-in-the-loop.

\subsection{Specification Updates Based on Atomic Propositions for Observation, Communication, and Control}
\label{sec:commObsControl}

Here we introduce the atomic propositions $\pi_i^o$, $\pi_{ij}^c$, and $\pi_i^u$ in (\ref{eq:globalSpec}) for each robot $i$, which correspond to observation, communication, and control. All these atomic propositions (APs) are dynamically checked and updated based on the robots' sensing. A similar approach has been used in decentralized multi-robot tasking in~\cite{filippidis2012cdc}, and here we extend it to distributed multi-robot systems that must meet a global specification. The observation proposition $\pi_{ij}^o$ for robot $i$ is true if an obstacle $j$ in workspace
cell $w_j$ is within its sensing range $r_i$ and false otherwise:
\begin{eqnarray}\label{eq:AP_comm}
\pi_{ij}^o(t)=\left\{
\begin{array}{l l}
\|x_i(t)-x_o^j\|\le r_i & \quad \textrm{true}\\
\|x_i(t)-x_o^j\|> r_i & \quad \textrm{false}
\end{array} \right.,~j\in{Obs}, \nonumber
\end{eqnarray}
where $x_o^j$ represents the actual position of an obstacle $j$.
When $\pi_{ij}^o$ is true, robot $i$ senses an obstacle $j$ along its planned path.
The robot then updates its model of the workspace by modifying its transition system representation $TS_i$ to label the state representing workspace cell $w_j$ with proposition $\pi^b_j$. The robot then resynthesizes its plan using the updated version of $TS_i$ according to the procedure described in Section \ref{subsec:symbolicMotionPlanning}, which then guarantees the specification $\Box(\pi_{ij}^o\to \lnot \bigcirc\pi^b_j)$.

The communication proposition $\pi_{i}^c$ for robot $i$ is true if another robot $j$ is within its communication range $\rho_i$ and false otherwise:
\begin{eqnarray}\label{eq:AP_comm}
\pi_{i}^c(t)=\left\{
\begin{array}{l l}
\|x_i(k)-x_j(t)\|\le\rho_i & \quad \textrm{true}\\
\|x_i(k)-x_j(t)\|>\rho_i & \quad \textrm{false}
\end{array} \right.,~j\ne i,~=1,2,\cdots, N. \nonumber
\end{eqnarray}
When $\pi_{i}^c$ is true, robots $i$ and $j$ can communicate with each other to exchange goal assignment, obstacle, and planned path information, allowing them to learn features of the environment they have not yet explored themselves and resynthesize their plans to avoid obstacles if necessary. This information can also be used to detect possible collisions between the two robots. 

We next introduce the control proposition. When $\pi_i^u$ is true, robot $i$ is executing the nominal LQR control law~(\ref{eq:AP_control}); when false, the robot pauses or replans its path autonomously:
\begin{eqnarray}
&&\hspace{-0.15in}u_i(t)=\left\{
\begin{array}{l l}
\textrm{LQR} & \quad \pi_i^u(t)\\
\textrm{wait or replan} & \quad \neg\pi_i^u(t)
\end{array} \right..
\nn \end{eqnarray}
When the communication proposition $\pi_{i}^c$ is true, robot $i$ has detected a potential collision
and communicates its path with involved robot $j$.
At this moment, the control proposition $\pi_i^u$ is set to false, and the robot takes one of two actions dependent upon the collision type to guarantee the specification $\Box(\pi_{i}^c\to \lnot \bigcirc\pi_i^u)$. Both waiting and replanning can be accomplished autonomously. 

Fig. \ref{fig:collision} illustrates  examples for possible collision scenarios. The robots' communication ranges are represented using dashed circles and their positions by solid circles. The arrows represent the path information communicated among the robots. We outline the collision avoidance method using the APs for communication and control (i.e., $\bigwedge_{i=1} ^{N}\Box(\pi_i^c\to\lnot\bigcirc\pi_i^u)$) as follows. 
\begin{assumption}\label{assump:smallerID}
Under the collision avoidance scenarios, without loss of generality, we assume that a robot with 
the smaller index $i \in I_R$
will always have higher priority in replanning and stop first if necessary. Similar prioritization policies can be imposed and the rest of the method will not be affected. 
\end{assumption}
Fig. \ref{fig:collision}(a) demonstrates the collision scenario when two robots will enter into a same cell at the next step. In this case, Robot 1 will stop and wait for Robot 2's next move according to our prioritization policy. Fig. \ref{fig:collision}(b) shows the scenario when two robots are at adjacent cells and will run into each other at the next step. In this case, both robots will replan to avoid collision by taking the other robot as an obstacle. Fig. \ref{fig:collision}(c) illustrates the scenario when three robots will run into a same cell at the next step. In this case, Robot 1 will stop first and Robot 2 and 3 will either stop or replan based on their collision type (i.e., the scenario shown in either Fig. \ref{fig:collision}(a) or  \ref{fig:collision}(b)). 
Note that this simple protocol is not guaranteed to be deadlock-free in arbitrary environments. For instance, if two robots must pass by each other through a 1-cell wide ``corridor'' of cells between two sets of obstacles and there is no other path to their goals, they will not be able to reach them unless they trade goals, leading to deadlock. Even if trading goals between robots is enabled, then over the course of multiple collision avoidance replanning events, robots can potentially trade and re-trade goals in such a way that some set of the goals are not ever reached, leading to livelock. In Section \ref{sec:liveness}, we present a more sophisticated protocol that is guaranteed to be both deadlock- and livelock-free.

\begin{figure}[thpb]
	\center
	\includegraphics[width = 5.5in]{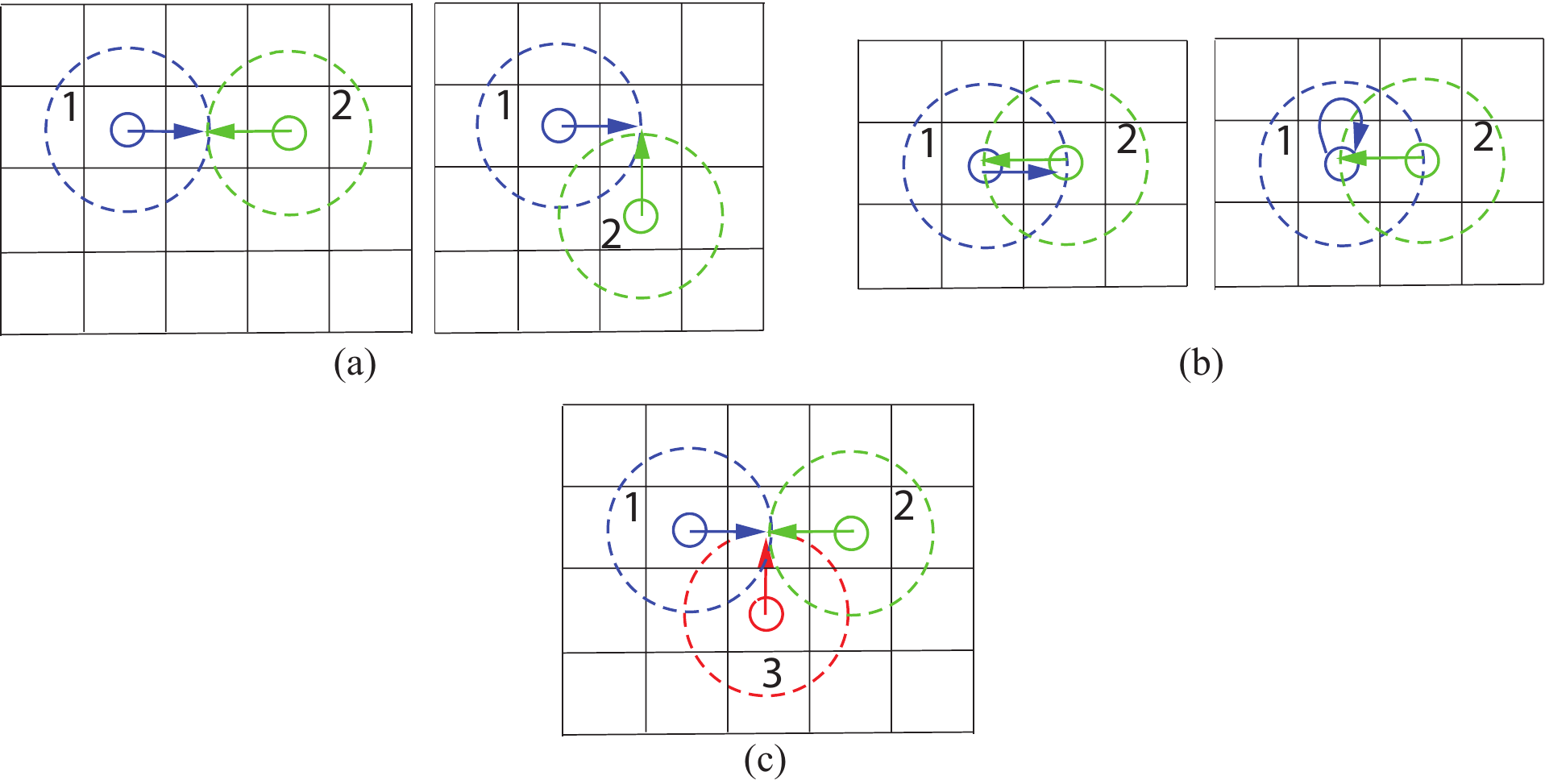}
	\caption{Collision avoidance scenarios.}
	\label{fig:collision}
\end{figure}


\subsection{Compositional Reasoning for Reachability and Obstacle Avoidance}
\label{sec:compositionalReasoining}

A general compositional reasoning approach can be used to show that the robots are able to collectively fulfill the reachability and obstacle avoidance portions of the global specification using a distributed planning approach. Note that this reasoning neglects the possibility of inter-robot collisions and switching to human control, which are addressed separately in Sections \ref{sec:commObsControl}, and \ref{sec:realTimeTrustBasedSwitching}, and an overall solution will be provided in Section \ref{sec:liveness}. Compositional reasoning in this context relies on concepts of interleaving of transition systems and unconditional fairness, which are defined as follows.

\begin{definition}[Interleaving of Transition Systems]
	\label{def:interleaving}
	The interleaving of two transition systems $TS_i= (S_i, Act_i,\delta_i, I_i, \Pi_i, L_i)$,  $i=1, 2$ is defined as
	$TS_1 \compose TS_2 = (S_1 \times S_2, Act_1 \cup Act_2, \delta, I_1 \times I_2, \Pi_1 \cup \Pi_2, L)$, 
	where transition relation $\delta$ is $(s_1,s_2) \rightarrow (s_1', s_2)$ if $s_1 \trans{\alpha} s_1'$ for $\alpha \in Act_1$ and 
	$(s_1,s_2) \rightarrow (s_1, s_2')$ if $s_2 \trans{\alpha} s_2'$ for $\alpha \in Act_2$, and labeling function $L$ is $L(s_1,s_2) = L(s_1) \cup L(s_2)$.
\end{definition}

\begin{definition}[Unconditional Fairness]
	A path $\sigma = s_0 \overset{\alpha_0}\rightarrow s_1 \overset{\alpha_1}{\rightarrow} s_2 \overset{\alpha_2}{\rightarrow} \ldots$ of a transition system $TS$ is unconditionally fair with respect to a set of actions $A \subseteq 2^{Act}$ if it contains infinitely many instances of each action $\alpha \in A$. 
	A transition system $TS$ satisfies $\varphi$ under an unconditional fairness assumption $\mathcal{F} \subseteq 2^{Act}$, denoted $TS \models_{\mathcal{F}} \varphi$, if all paths in $TS$ that are unconditionally fair with respect to $\mathcal{F}$ satisfy $\varphi$.
	\label{def:fairness}
	
\end{definition}

Recall that, neglecting inter-robot collisions and switching to manual motion planning, each robot $i$ will eventually complete execution of a specific path $\sigma^i$ that reaches its assigned goals while avoiding all obstacles according to the procedure and reasoning given in Section \ref{subsec:symbolicMotionPlanning}. Note that this can be seen as a special case of a transition system comprising only one path. Let us then represent this path as transition system $TS_{\sigma}^i$ and note that $TS_{\sigma}^i \models  \bigwedge_{j\ \in Goals_i} \lozenge \pi_j^g \land  \bigwedge_{i=1} ^{N}\bigwedge_{j \in Obs}\Box(\pi_{ij}^o\to \lnot \bigcirc\pi_j^b)$.
Assuming each robot can take an infinite number of actions over time -- a natural assumption since the robots act independently -- we have the following lemma.
\begin{lemma}
	Assuming there exists at least one path to each assigned goal, there are no inter-robot collisions, and there is no switching to manual motion planning, the robots are guaranteed to eventually reach all assigned goals while avoiding collisions with obstacles under an assumption of unconditional fairness. 
	\label{lemma:indvSpec}
\end{lemma}
\begin{proof} 
The combination of all robots acting in the workspace together is $TS_{R_W} = TS_{\sigma}^1 \compose \cdots \compose TS_{\sigma}^N$. 
	Assuming the robots act independently, each robot is able to take an infinite number of actions, 
	and all actual execution paths of $TS_{R_W}$ will be unconditionally fair with respect to $\mathcal{F} = {\cup_{i=1}^N Act_i}$.
	
	Note that every path in $TS_{R_W}$ is formed by interleaving the individual paths represented by each $TS_{\sigma}^i$  where $\sigma_i \models \bigwedge_{j \in Goals_i} \lozenge \pi_j^g$. Then $TS_{R_W} \models_{\mathcal{F}} \bigwedge_{j \in Goals} \lozenge \pi_j^g$, since all actions $\alpha^i \in Act_i$ that enable each $TS_{\sigma}^i$ to reach states satisfying each goal $\pi_j^g \in Goals_i$ will be executed, and $\bigcup_{i \in I_R} Goals_i = Goals$. Furthermore, since all robots satisfy $TS_{\sigma}^i \models \bigwedge_{j \in Obs} \Box (\pi_i^c\to\lnot\bigcirc\pi_j^b)$, then $TS_{R_W} \models \bigwedge_{i=1}^{N}\bigwedge_{j \in Obs} \Box (\pi_i^c\to\lnot\bigcirc\pi_j^b)$, and there are no states for any $TS_{\sigma}^i$ in which $\pi_j^b$ for $j \in Obs$ holds. 
\end{proof}

\begin{remark}
	The trust belief estimate based on Eq. (\ref{eq:BayesUpdates}) will determine the specification decomposition, with more trusted robots assigned more destinations. Since trust is dynamically involving, this robot assignment will be updated on the fly. $\hfill\bullet$
\end{remark}

\section{
	Real-Time Trust-Based Switching Between Manual and Autonomous Motion Planning}
\label{sec:realTimeTrustBasedSwitching}

In this section, we utilize trust analysis in a real-time switching framework to enable switches between manual and autonomous motion planning. Recall from Section \ref{sec:humanRobotInteraction} that although the autonomous motion planning is guaranteed to be correct, it is usually conservative due to overapproximation of the environment. So while the autonomous motion planning is safe, more efficient but riskier paths -- in this case, paths between obstacles in adjacent regions -- may exist. If the human trusts a robot's ability in navigating between two obstacles, the human can choose to construct a more efficient path between the obstacles based on, e.g., sensory information about the obstacles supplied by the robot. 

\begin{remark}
Under the autonomous motion planning mode, to guarantee the bisimulation between high-level discrete motion planning and low-level continuous control, we utilize a nominal LQR controller (\ref{eq:AP_control}), and hence a robot moves between centroids of adjacent cells. That is to say, the cells are assumed to be four-cornered and a robot can only move ``left", ``right", ``up", and ``down" at the next step. However, under the manual motion planning mode, there is no such restriction and a robot can move in any direction following the human assigned waypoints. $\hfill\bullet$
\end{remark}

Fig. \ref{fig:trustPaths} shows an example of how trust can benefit the motion planning of an autonomous robot, generated by integrating the NuSMV model checker with Matlab. The environment of the system is represented in Fig. \ref{fig:trustPaths}. The robot begins in cell 5, represented by the circle, and must reach the goal in cell 15, represented by the diamond.
The obstacles are represented in the continuous space by filled polygons 
and in the discrete space by cells marked with ``X''s.
As Fig. \ref{fig:trustPaths}(a) shows, the autonomous motion planner avoids regions marked by Xs and generates a safe but lengthy path to the goal;
while there is a gap between the two obstacles through cell 10, the overapproximation prevents the robot from generating a path through this region. If the human operator trusts the robot's ability to follow a path between the obstacles and is not overloaded,
the manual motion planning mode will be activated and the benefits can clearly be seen in Fig. \ref{fig:trustPaths}(b). This tradeoff between safety and efficiency motivates the use of trust-based switching between manual and autonomous motion planning. 

\begin{figure}[h]
	\centering
	\includegraphics[width=3.5in]{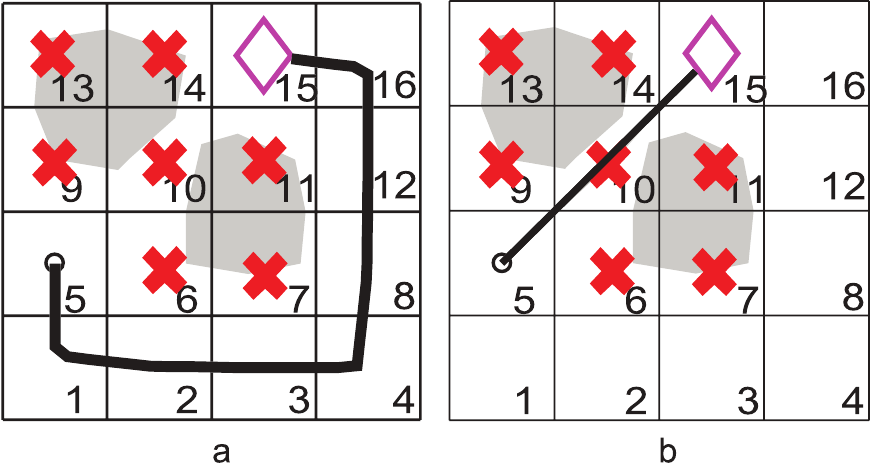}
	\caption{(a) Safe autonomous motion planning in low trust scenario, and (b) advanced manual motion planning in high trust scenario.}
	\label{fig:trustPaths}
\end{figure}

This high trust path, however, is inherently risky since information regarding the environment in a typical scenario is limited to what the robot has sensed and provided to the human through visual feedback.  
It is therefore necessary to develop a monitoring and switching framework such that the autonomous safe motion planning can be reactivated if any new obstacle in path is detected and a collision is about to happen under the path planned by the human. The overall architecture of our proposed framework is based on two frameworks which were originally designed for augmenting system safety. The first framework is called the Simplex architecture, which guarantees application-level safety by ``using simplicity to control complexity''~\cite{sha2001using,bak2009system}, and the second one is called the Monitoring and Checking (MaC) framework, which ``bridges the gap between formal specification and verification" and was developed originally to assure the correctness of program execution at runtime~\cite{kim1998framework}. Fig.~\ref{fig:trustDiagram} shows the schematic of the proposed monitoring and trust-based switching framework.
This framework is logically divided into five subsystems: Motion Planner, Controller, Monitor, Checker, and Decision Maker. This system is designed to be able to control the robot in two modes: Manual versus Autonomous. The advanced sub-system, i.e., manual mode, is less safe since it allows the human to plan riskier paths. The baseline subsystem refers to the autonomous mode and uses the symbolic motion planner, which is guaranteed to be correct and hence safe. In the Monitor subsystem, we have two modules: Filter and Event Recognizer. The filter is designed to extract the motion state information of a robot and send it to the event recognizer. The event recognizer detects an event from the values received from the filter based on event definitions provided by a monitoring script. The monitoring script maps the system states to the events at the requirement level to enable analysis by the system checker. Here, events are defined according to the condition $\|x_i - x_o^j\| \leq r_o$, where $x_o^j$ is the obstacle position and $r_o$ is some minimally acceptable distance between the robot and the obstacle. Should the robot come within this distance, an event will be detected. According to the mean trust equation (\ref{eq:trust}), this event detection also leads to fault penalty and hence lowers trust in the robot, leading to a re-evaluation of the assigned tasks. The result is that other more trusted robots may be re-assigned some of the destinations that were originally assigned to the
robot that generated the fault.
Once this re-evaluation is performed, the operator is free to continue working with the same or another robot depending on the change of levels of trust. Once the event recognizer detects an event, it will send the information to the checker module. The checker checks whether or not the current execution of the system meets the specification. Based on the information received from the checker, the decision module determines under which mode the system should run for motion planning. More specifically, the trust-based decision module first computes the trust belief distribution based on Equation (\ref{eq:BayesUpdates}) for each robot $i$, then the corresponding trust value that yields the maximum likelihood is obtained. Next, the current maximum likelihood trust is compared with the maximum likelihood trust at the previous time step and the change of trust value can be calculated. The decision module will suggest that the human collaborate with the robot that has the highest trust increase beyond a certain threshold. In case multiple robots have the same highest trust change, some priority criterion can be used to choose an individual robot. See Fig. \ref{fig:GUI} for an example of the GUI design used in our simulation for a robot requesting manual motion planning based on trust comparison.

\begin{figure}[ht]
\centering
\resizebox{4.0in}{!}{\includegraphics{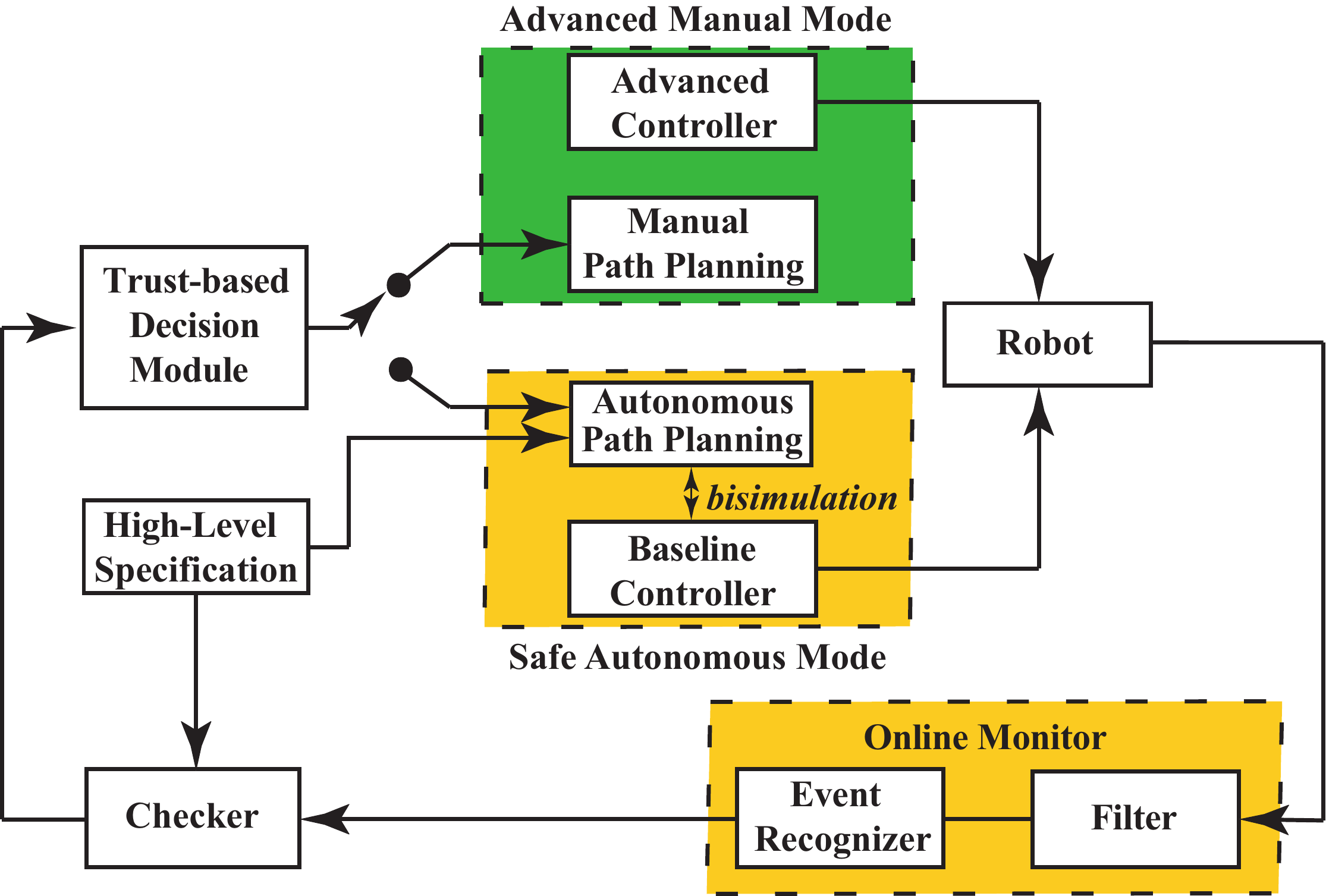}}
\caption{The overall structure of the monitoring and trust-based switching framework}\label{fig:trustDiagram}
\end{figure}



\begin{remark}
	From the proof of Lemma \ref{lemma:indvSpec} it is shown that given no collision among robots, each robot will generate a path that will eventually guarantee reaching its goals, with the bisimulation property given in Section \ref{sec:humanRobotInteraction} guaranteeing the robot will follow this path correctly. 
	The human generated plans provide a more efficient way of reaching goals, but can be risky if there is incomplete environmental information. However, using the monitoring and switching framework outlined above, the robot is guaranteed to never collide with an obstacle and will return to the proven safe method. This switching will continue until all goals are reached or a complete obstacle map is generated. Assuming that the human is guaranteed to never permanently take control of one of the robots, the fairness assumption used in Lemma \ref{lemma:indvSpec} is satisfied. 
	$\hfill\bullet$
\end{remark}

\section{Deadlock- and Livelock-free Algorithms}\label{sec:liveness}

As discussed briefly in Section \ref{sec:commObsControl}, the previously described collision avoidance protocol does not prevent deadlock, e.g., cases in which two or more robots cannot move toward their goals because they are blocking each other's paths. Naively allowing robots to trade goals in such a case could result in livelock situations~\cite{raman2013towards}, in which robots could be cyclically re-trading goals over multiple encounters while never actually reaching them. To address the problem, we propose the following deadlock- and livelock-free algorithms, which collectively guarantee the reachability property $\bigwedge_{j\ \in Goals} \lozenge \pi_j^g$, i.e., that all goals are eventually visited by a robot.

Let each robot $i \in I_R$ have a ``current'' goal $g_i \in Goals$ and a set of ``next'' goals 
$Goals'_i \subset Goals$ such that $g_i \cap Goals'_i = \emptyset$, 
$Goals'_i \neq \emptyset \rightarrow g_i \neq \emptyset$, and the union of all current and next goals for all robots is $G_R$ = $\bigcup_{i \in I_R} (g_i \cup Goals'_i)$, with $G_R = Goals$ at initialization. 
Let us also suppose that $(g_i \cup Goals'_i) \cap (g_j \cup Goals'_j) = \emptyset$ for all $i, j \in I_R$ where $i \neq j$. 
This last condition is not strictly necessary to ensure that all goals are eventually reached, 
but it helps avoid duplicate visits to goals by multiple robots and the need to remove a goal from a robot's set of goals  
if another robot is observed at the goal location. 
Denote the planned path for robot $i$ as $\sigma^i =s_0^i \overset{\alpha_0^i}{\rightarrow} s_1^i\overset{\alpha_1^i}{\rightarrow}s_2^i\ldots$
and let $\sigma^i[k..]$ denote the suffix of $\sigma^i$ starting at $k$, i.e., $s_k^i \overset{\alpha_k^i}{\rightarrow} s_{k+1}^i\overset{\alpha_{k+1}^i}{\rightarrow}s_{k+2}^i\ldots$.
Then deadlock- and livelock-free algorithms that ensure all the goals are reached under a set of mild assumptions are as follows. 

\begin{algorithm}
\caption{Deadlock- and livelock-free, human-in-the-loop symbolic motion planning $\forall i \in I_R$}
\SetAlgoNoLine
\KwIn{$Goals'_i \subset Goals$ and $g_i \in Goals$ such that $g_i\cap Goals'_i = \emptyset$, $Goals'_i\neq \emptyset \rightarrow g_i \neq \emptyset$, and the union of all current and next goals for all robots $G_R$, with $G_R = Goals$ at initialization.}
\KwOut{Deadlock- and livelock-free, human-in-the-loop symbolic path planning for every $i\in I_R$}
\For{i++}{
	Plan Path(); \textbf{//ALGORITHM 2}\
	
\Repeat{false}{
	Follow Path(); \textbf{//ALGORITHM 3}
}
}
\end{algorithm}

\begin{algorithm}
\caption{Plan Path()}
\SetAlgoNoLine
\eIf{$g_i \neq \emptyset$ and no other robots are in communication range}{
      Plan a path $\sigma^i$ to $g_i$\;
      }
{\If{$g_i \neq \emptyset$ and other robots $I \subseteq I_R$ are in communication range}{
      Ensure that for $g^* = \min\{g_i \mid i \in I \}$, the currently assigned robot $i$ can reach it given the position of the other robots, or it is assigned to the robot $j$ closest to $g^*$ such that $\sigma^j = \sigma^i[k..]$ for some $k > 0$\;
      Reassign all other goals in $\bigcup_{i \in I} (Goals'_i \cup g_i) \setminus g^*$ to robots in $I$;
     }
     }
\end{algorithm}

\begin{algorithm}
\caption{Follow Path()}
\SetAlgoNoLine
\Repeat{Obstacles or other robots in range, or human intervention}{
Drive robot using the LQR law (\ref{eq:AP_control})\;
}
\If{Obstacle in observation range $r_i$}{
   Record obstacle\;
   \If{Obstacle in path}{
      Plan Path(); \textbf{//ALGORITHM 2}
	}
    Follow Path(); \textbf{//ALGORITHM 3}
}
\If{Other robots in communication range $\rho_i$}{
	Share obstacle and path data\;
    Plan Path(); \textbf{//ALGORITHM 2}\
    
    Follow Path(); \textbf{//ALGORITHM 3}
}
\If{Human control requested and granted}{
   \Repeat{Human is not in control}{
      \If{Obstacle too close}{
         release human control\;
		break\;
		}
      \If{Exited obstacle region}{
          Release human control\;
          break\;
     }
     Plan Path(); \textbf{//ALGORITHM 2}\
     
     Follow Path(); \textbf{//ALGORITHM 3}
}
}
\If{$g_i$ reached}{
   \If{$Goals'_i \neq \emptyset$}{
	Set $g_i$ to a goal in $Goals'_i$\;
	Set $Goals'_i$ to $Goals'_i \setminus g_i$\;
	}
   \If{$Goals'_i = \emptyset$}{
      break;
    }
}
\end{algorithm}

The assumptions and proof of correctness are as follows.

\begin{assumption}
There is an obstacle-free path of finite length between every pair of obstacle-free cells.
\label{a:obstacleFree}
\end{assumption}

\begin{assumption}\label{assumption:comm}
For every robot $i \in I_R$, the communication range $\rho_i$ is greater than the size of any single cell. 
\label{a:commRange}
\end{assumption}

\begin{assumption}
The human operator can only choose to control a robot if the robot requests to move through a set of obstacles. 
Human control of the robot ends in finite time.
\label{a:humanControl}
\end{assumption}

\begin{assumption}
Every robot executes actions infinitely often, where remaining in the current cell for a finite amount of time is considered valid action. All actions complete in finite time.
\label{a:fairness}
\end{assumption}

\begin{lemma}
Every robot $i$ is able to plan a path $\sigma^i$ to its current goal $g_i$ such that the number of actions needed to reach $g_i$ is finite, thus preventing deadlock.
\label{lemma:generalProgress}
\end{lemma}

\begin{proof} 
This follows from Assumption \ref{a:obstacleFree} in the case that the planned path does not include any obstacle cells, additionally from Assumption \ref{a:humanControl} if the planned path does include obstacle cells due to the human engaging manual planning, and the fact that when multiple robots are in communication range, the inter-robot collision protocol never assigns a current goal $g_i$ to a robot that cannot reach it given the current position of the other robots. 
\end{proof}

\begin{lemma}
Whenever two or more robots with indices $I \subseteq I_R$ cooperatively replan their paths according to the inter-robot collision protocol, if the goal $g^* = \min \{g_i \mid i \in I\}$ is reassigned, the number of actions needed by the newly assigned robot to reach it decreases, thus preventing livelock.
\label{lemma:collisionProgress}
\end{lemma}
\begin{proof} 
Due to Assumption \ref{a:obstacleFree}, if the robot $i$ assigned to $g^*$ is unable to continue on its planned path to $g^*$, it can only be because one or more other robots occupy cells on this path. The ``Path Plan" algorithm reassigns $g^*$ to the robot $j$ that requires the least number of cell transitions to reach $g^*$ along the original path of robot $i$. 
By Assumption \ref{a:obstacleFree}, there exists a valid path $\sigma^j = \sigma^i[k..]$ for some $k > 0$, and the number of actions in the path $\sigma^j$ to the goal cell $g^*$ is less than the original path $\sigma^i$.
\end{proof}

\begin{theorem}[Reachability]
Every goal is eventually reached by a robot.
\label{theorem:reachability}
\end{theorem}
\begin{proof}
Initially, let $G_R$ be the union of all current and next goals for all robots, i.e., $G_R = \bigcup_{i \in I_R} (g_i \cup Goals_i) = Goals$. The proof is by induction.

\noindent Base case: $g_i = \min(G_R)$ for some robot $i$. Either robot $i$ reaches $g_i$ by Lemma \ref{lemma:generalProgress} or some other robot reaches it by Lemma \ref{lemma:collisionProgress}. $G_R$ is then set to $G_R\setminus g_i$, decreasing the size of $G_R$ by 1.

\vspace{0.5em}\noindent Induction: $g_i > \min(G_R)$ for some robot $i$. Either robot $i$ reaches $g_i$ by Lemma \ref{lemma:generalProgress}, $G_R$ is then set to $G_R \setminus g_i$, and the size of $G_R$ reduces by 1, or the robot enters the collision avoidance protocol with a set of robots and one of those other robots continues progress on $g_i$ or some $g_j < g_i$ by Lemma \ref{lemma:collisionProgress}.

\vspace{0.5em}\noindent Since $G_R$ is finite, eventually $G_R = \emptyset$ and all goals are reached.
\end{proof}

%
%

Under the above deadlock- and livelock-free algorithms, we can resolve collision avoidance among 3 robots by goal reassignment, a problem we could not previously solve using a more simple protocol. We illustrate the solution approach as follows. 

Fig. \ref{fig:case4} shows two example collision scenarios among 3 robots. Robots 1, 2, 3 are within each other's communication range and hence can exchange obstacle, planned path, as well as goal assignment information. As illustrated in the figure, Robot 1 and 2 will run into the same cell at the next move (i.e., the collision type depicted in Fig. \ref{fig:collision}(a)); Robot 2 and 3 are at adjacent cells and will run into each other at the next move (i.e., the collision type depicted in Fig. \ref{fig:collision}(b)). In this case, simply letting one of the robots wait and the other two replan might trap the robots in a deadlock situation. Instead, we can exchange the robots goals (i.e., assign the goals to robots that travel from cells closest to them) as in algorithms 1-3 to resolve the possible collision. Other similar collision scenarios among 3 robots can be resolved using the same path replan and goal reassignment methods.

\begin{figure}[h]
\center
\includegraphics[width = 0.5\columnwidth]{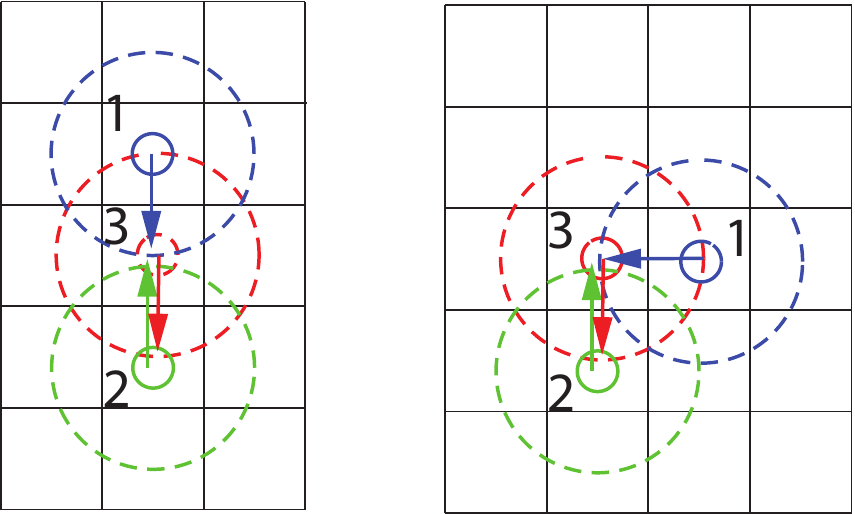}
\caption{Example collision scenarios among 3 robots: Robot 1 and 2 will run into a same cell at the next move; Robot 2 and 3 are at adjacent cells and will run into each other at the next move.}
\label{fig:case4}
\end{figure}


Fig. \ref{fig:case5} shows another two example collision scenarios among 3 robots. At step a, Robot 2 and 3 will run into a same cell at the next move (i.e., the collision type depicted in Fig. \ref{fig:collision}(a)); Robot 1 follows Robot 2 and there is no collision. According to Assumption \ref{assump:smallerID}, a robot with smaller index number will stop first, therefore Robot 2 will stop and wait for Robot 3's next move. However under this policy, at step b, Robot 1 and 2 will collide with each other. In this case, the robots' goals will be reassigned according to the above algorithms to resolve the possible collision. Other similar collision scenarios among 3 robots can be resolved using the same methods.

\begin{figure}[h]
\center
\includegraphics[width = 0.7\columnwidth]{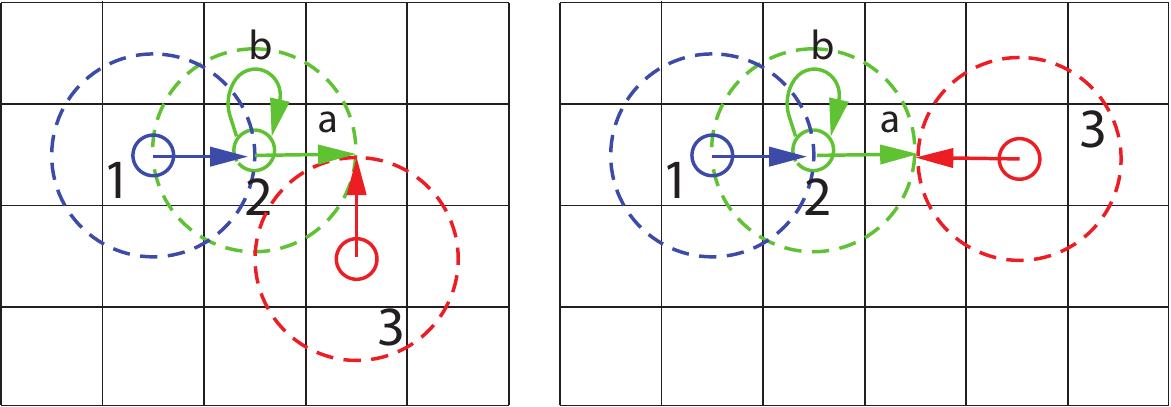}
\caption{Example collision scenarios among 3 robots: Robot 2 and 3 will run into a same cell at the next move and hence Robot 2 waits, however under this policy Robot 1 and 2 will collide at the next move.}
\label{fig:case5}
\end{figure}

Livelock is harder to visualize since it is not due to single collision avoidance event but rather the possibility of goals being cyclically reassigned over the course of multiple collision avoidance events. It is also rare since it would require just the right circumstances in regards to the positioning of obstacles and the exact reassignment of goals during each collision event. Theorem \ref{theorem:reachability} suffices to show that the algorithms are livelock-free.

\section{Simulation}
\label{sec:simulation}

In this section, a set of simulations of the ISR scenario of Section \ref{sec:humanRobotInteraction} is used to demonstrate our methods.
The simulation is conducted in Matlab with model checking performed using NuSMV. We show the simulation results for distributed symbolic motion planning for multiple autonomous robots under the proposed obstacle and inter-robot collision avoidance protocols in Section \ref{sec:sim_auto}, the human input devices, manual motion planning approach, and GUI designs for trust measurements in Section \ref{sec:sim_GUI}, and the overall trust-based motion planning strategy with a human-in-the-loop in Section \ref{sec:sim_human}. 

\subsection{Distributed Motion Planning for Multiple Autonomous Robots}\label{sec:sim_auto}
Fig. \ref{fig:progression4_keymoments} shows an example environment that consists of 3 robots with marked index numbers, 6 goals marked by diamonds, and 12 obstacles marked by crosses. The obstacles in the environment are initially unknown by the robots until they are gradually sensed. The sensor range $r_i$ of a robot is marked by a dashed circle around it. In our simulations, we set the range in a way such that a robot can always observe the 8 neighboring cells around it. Once an obstacle is sensed, its position becomes known to that individual robot. The communication range $\rho_i$ of a robot is set to be the same as the sensing range. The obstacle information, planned path, and goal assignment information are communicated with other robots when they come within communication range. The robot paths are demonstrated using bold line segments. The robot start positions are marked by rectangles and the robot current positions are marked by circles. 


\begin{figure*}[t]
	\center
	\includegraphics[width = 1\textwidth]{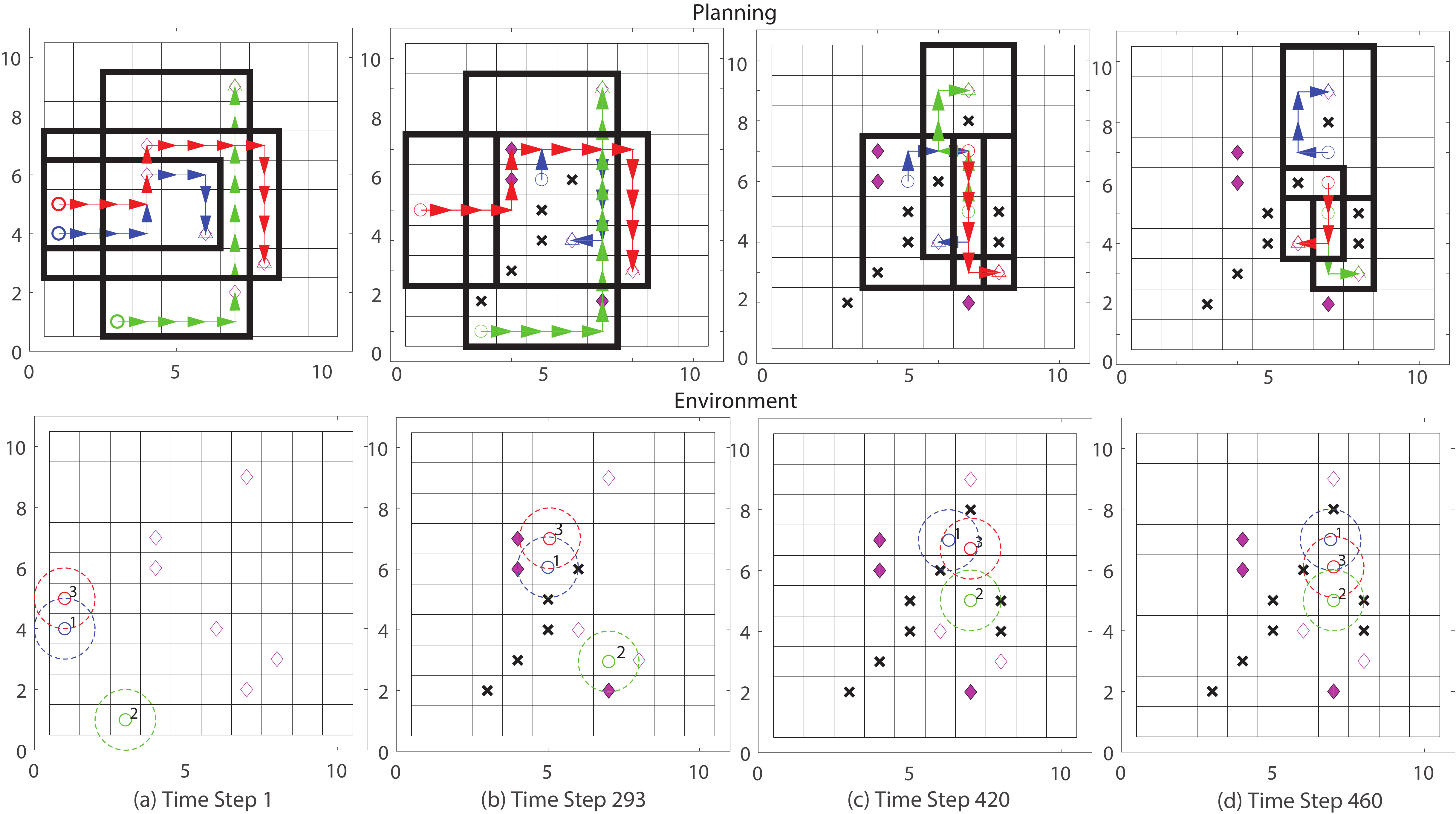}
	\caption{Progression of simulation of the inter-robot collision scenario in Fig. \ref{fig:case4} (a) at time step $t=1$ (initial plan and robot position), (b) $t=293$ (new obstacle at (6,6) detected in the path and Robot 1 replans), (c) $t=420$ (Robots 2 and 3 communicate each other's obstacle information and replan), and (d) $t=450$ (Robots 1,2, and 3 enter into the collision scenario in Fig.~\ref{fig:case4} and hence reassign goals and replan.) }
	\label{fig:progression4}
\end{figure*}

\begin{figure*}[t]
	\center
	\includegraphics[width = 0.45\textwidth]{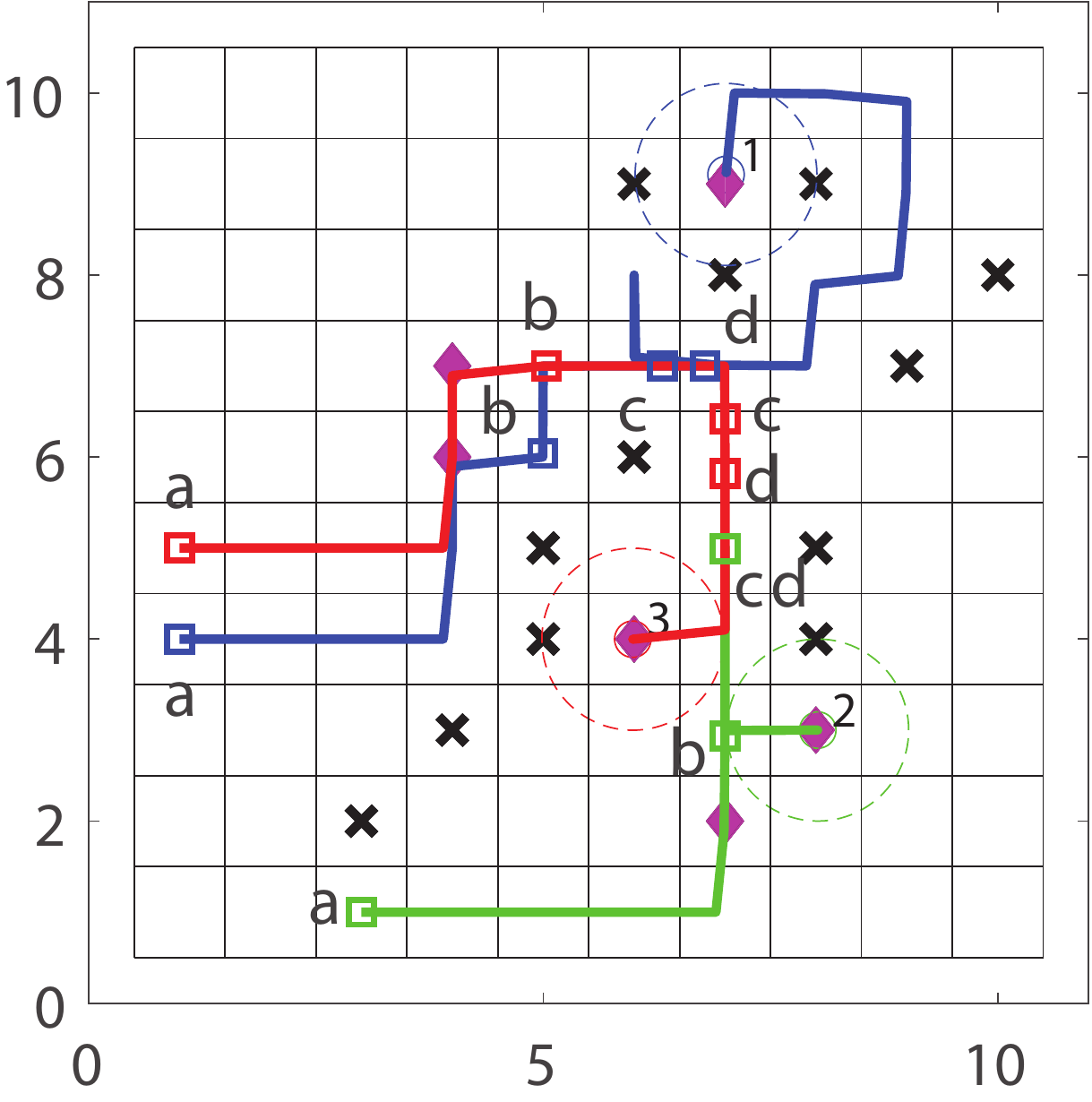}
	\caption{Final paths and key moments (marked by rectangles from step (a)-(d)) of each robot under the simulation scenario shown in Fig. \ref{fig:progression4}.}
	\label{fig:progression4_keymoments}
\end{figure*}

\begin{figure*}[t]
	\center
	\includegraphics[width = 1\textwidth]{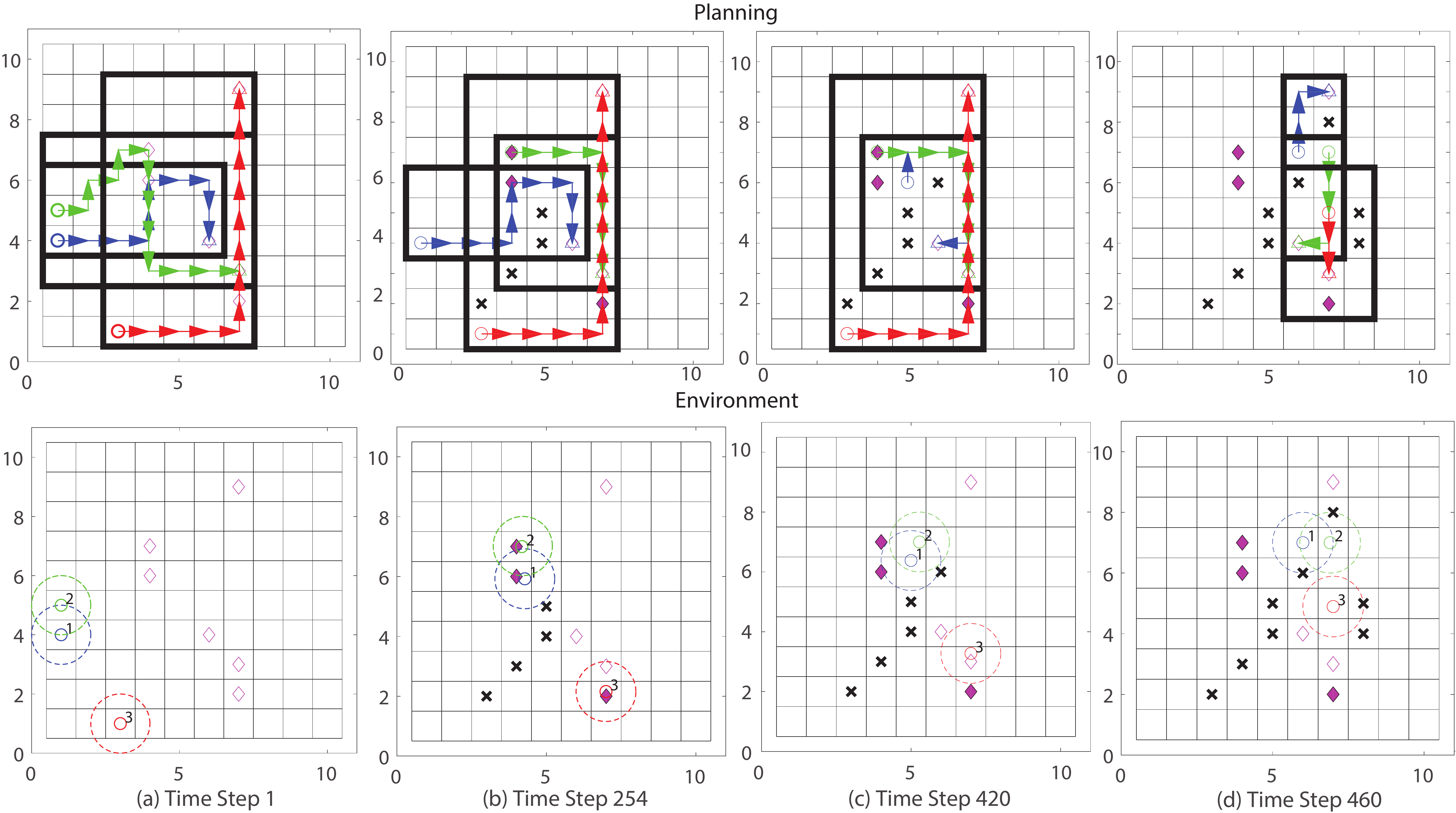}
	\caption{Progression of simulation of the inter-robot collision scenario in Fig. \ref{fig:case5} (a) at time step $t=1$ (initial plan and robot position), (b) $t=250$ (Robots 1 and 2 are within each other's communication ranges; Robot 2 finds new obstacles in the path and hence replans), $t=310$ (new obstacle at (6,6) detected in the path and Robot 1 replans), and (d) $t=391$ (Robots 1,2, and 3 enter into the collision scenario in Fig.~\ref{fig:case5} and hence reassign goals and replan.) }
	\label{fig:progression5}
\end{figure*}

\begin{figure*}[t]
	\center
	\includegraphics[width = 0.45\textwidth]{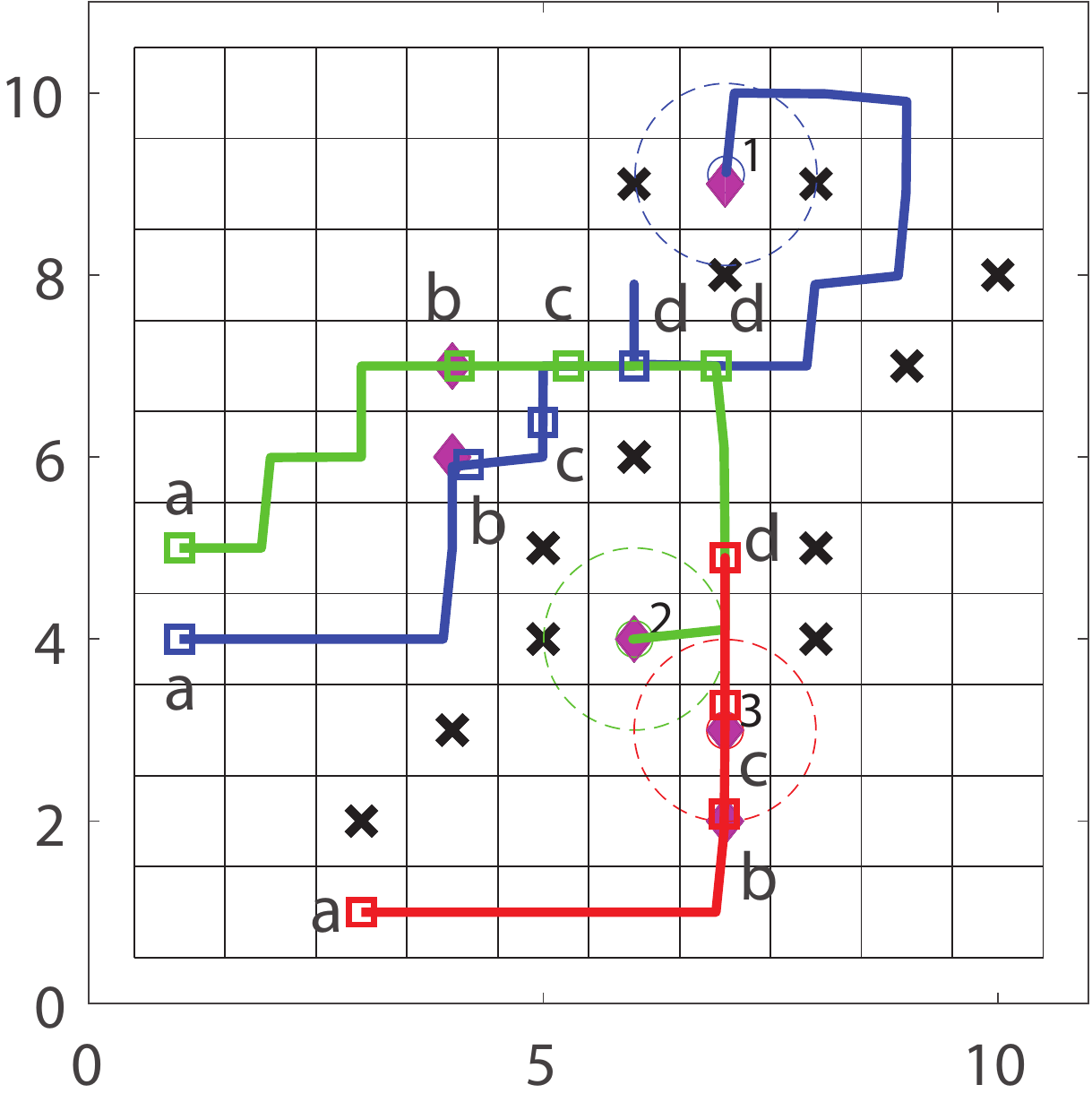}
	\caption{Final paths and key moments (marked by rectangles from step (a)-(d)) of each robot under the simulation scenario shown in Fig. \ref{fig:progression5}.}
	\label{fig:progression5_keymoments}
\end{figure*}

To find a path such that the local subspecification is satisfied, the symbolic motion planning approach described in Section \ref{subsec:symbolicMotionPlanning} is used by each individual robot. To address computational complexity, the compositional reasoning concept is used to decompose (\ref{eq:globalSpec}) into local subspecifications as described in Section \ref{sec:compositionalReasoining}. 
To further reduce computational complexity, each robot only computes its path over a local subset (see the black boxes in Figs. \ref{fig:progression4} and \ref{fig:progression5}). This subset is determined using the robot's knowledge of its assigned goals and the location of obstacles that have been sensed. 
Using this method, every path generated by an individual robot will end in either reaching its goals or detecting an obstacle and will continue until all obstacles are detected. To prevent collisions among the robots, the communication and control protocols detailed in Section \ref{sec:commObsControl} and the deadlock- and livelock-free algorithms are used whenever the robots are within each other's communication range. Therefore, the robots are guaranteed to eventually reach their goals. 

Fig. \ref{fig:progression4} shows the motion planning process to avoid the collision scenario as shown in Fig.~\ref{fig:case4}. Fig. \ref{fig:progression4_keymoments} shows the final path of each robot and marks the key moments (a)-(d) corresponding to time step $t=1$ to $t=450$ as illustrated in Fig.~\ref{fig:progression4}. Fig. \ref{fig:progression5} shows the motion planning process to avoid the collision scenario as shown in Fig.~\ref{fig:case5}. Fig. \ref{fig:progression5_keymoments} shows the final path of each robot and marks the key moments as illustrated in Fig.~\ref{fig:progression5}. The motion planning process for the collision scenarios depicted in Fig.~\ref{fig:collision} can be demonstrated in a similar fashion.

\subsection{Human Input Devices, Manual Motion Planning, and GUI Design}\label{sec:sim_GUI}


\begin{figure}[h]
\center
\includegraphics[width = 0.7\columnwidth]{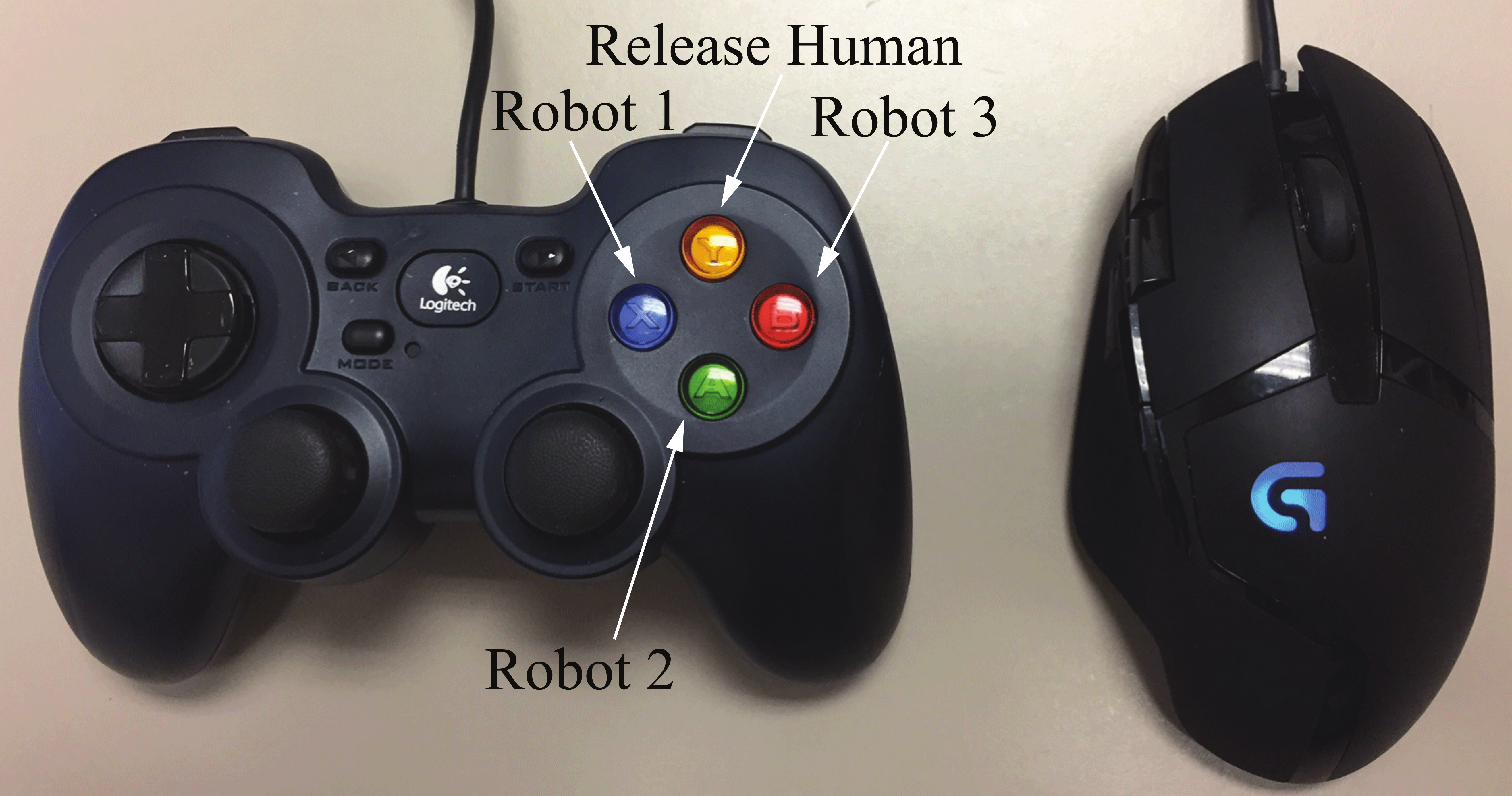}
\caption{Human input devices: gamepad and mouse. The gamepad buttons are used to select and confirm a robot to collaborate with, e.g., ``blue" button for Robot 1, ``green" button for Robot 2, ``red" button for Robot 3, and ``yellow" button for releasing the human. The mouse is used to give waypoints to a robot manually.}
\label{fig:humaninput}
\end{figure}

\begin{figure}[h]
\center
\includegraphics[width = 0.5\columnwidth]{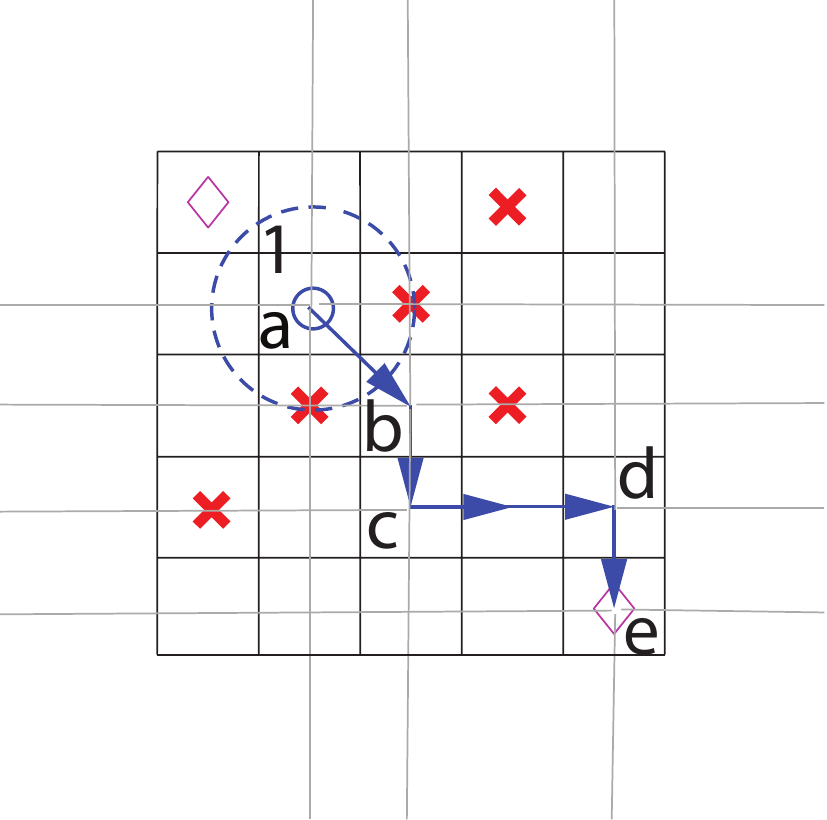}
\caption{Illustration of manual motion planning using a mouse. The human assigns waypoints starting from a, passing between two obstacles to point b, and then point c, d, and eventually reaching the goal e.}
\label{fig:manualpath}
\end{figure}

\begin{figure}[t]
        \centering
\begin{tabular}{c}
       \subfigure[]         
                {\includegraphics[width=4.5in]{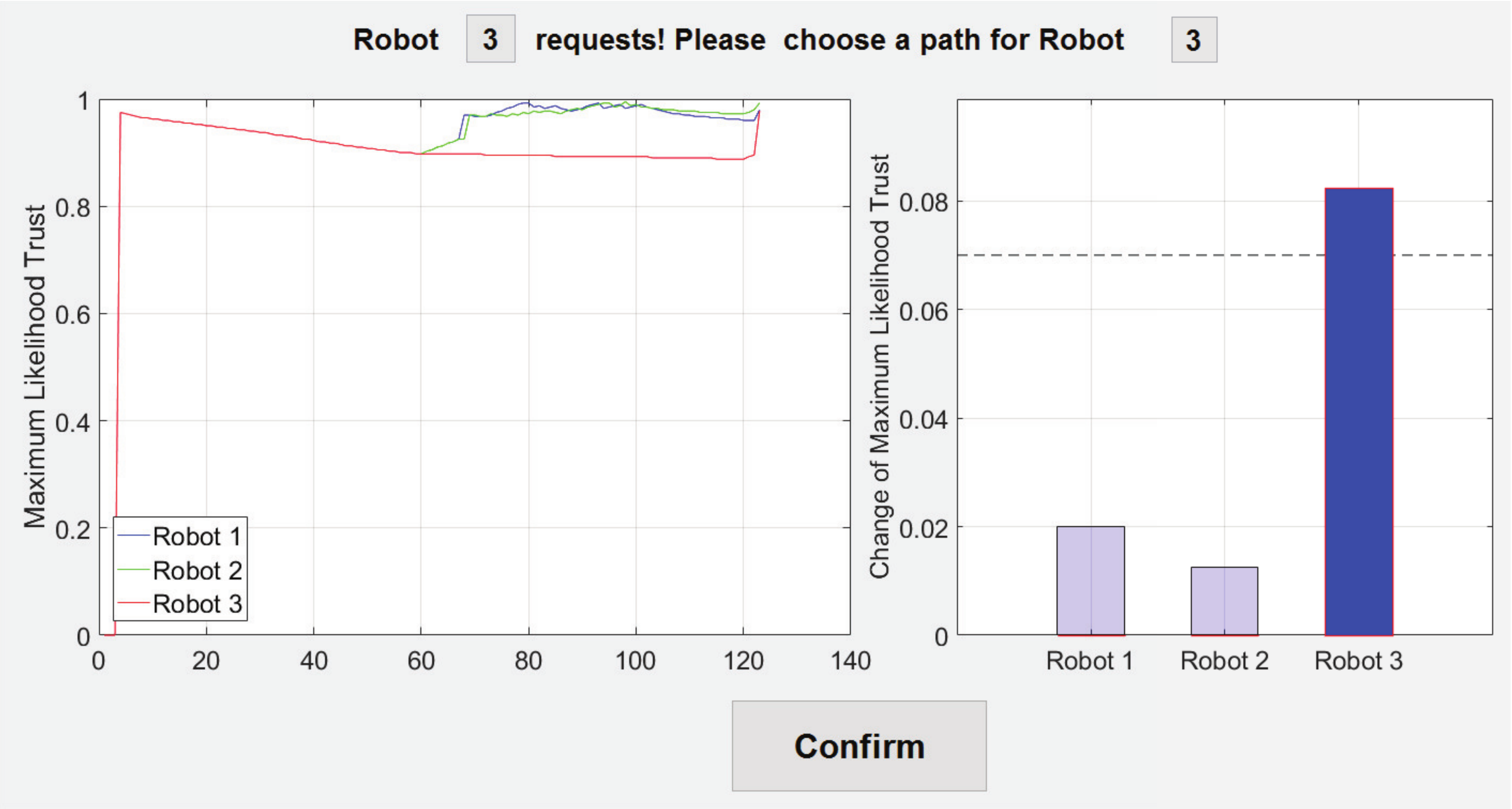}}\\
      \subfigure[]
                {\includegraphics[width=2.3in]{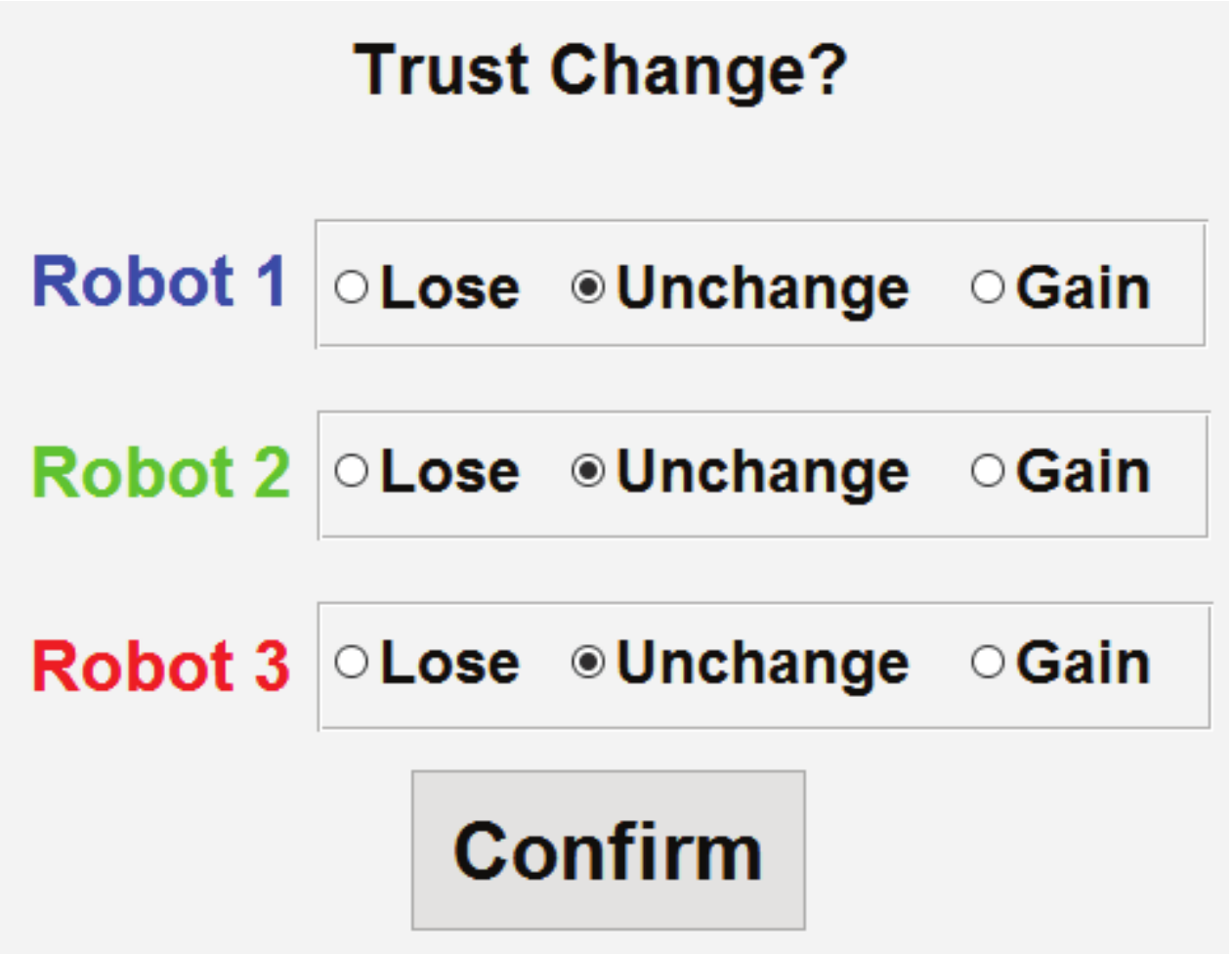}}\\
       \subfigure[]
                {\includegraphics[width=3.8in]{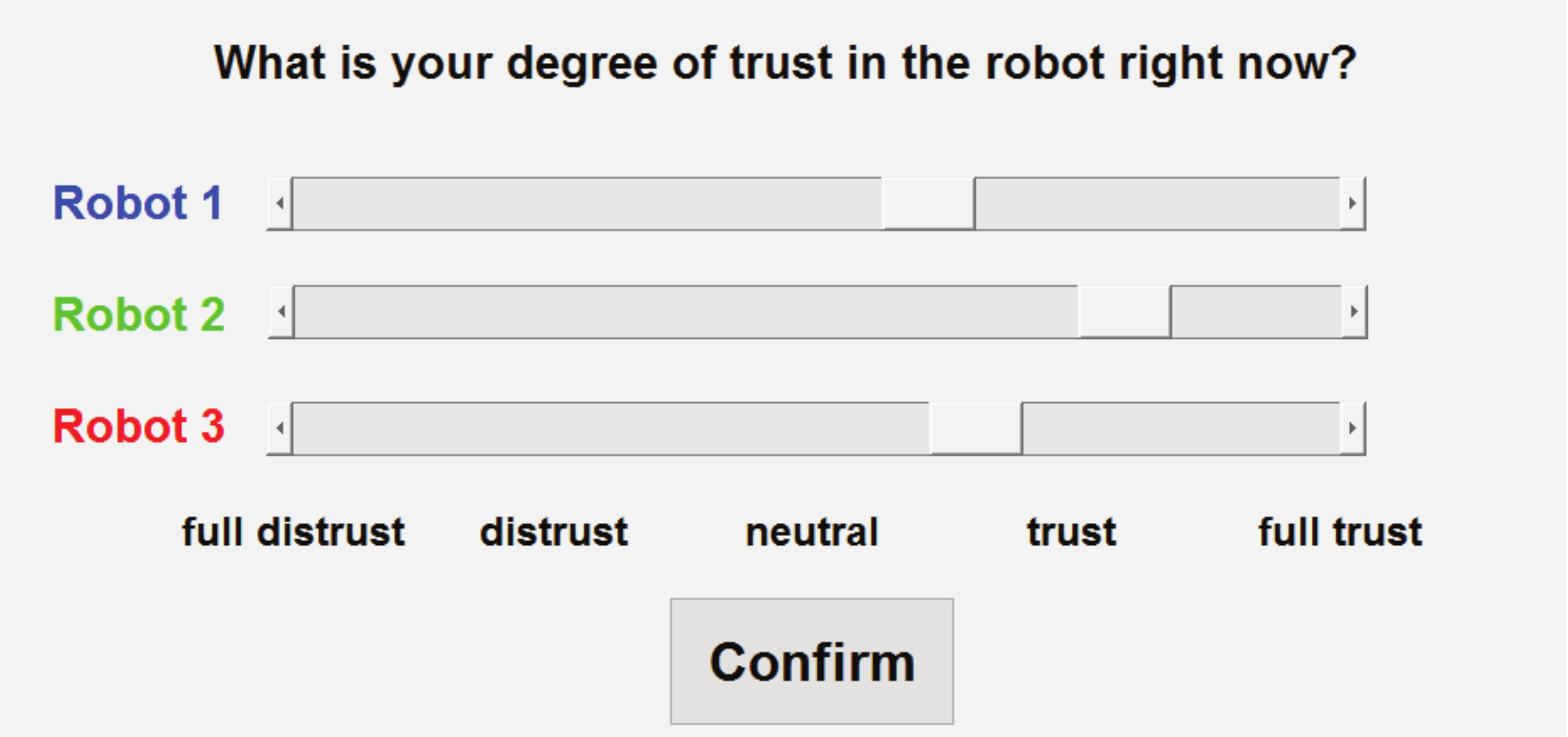}}
\end{tabular}
         \caption{GUI designs: (a) Comparison of robots' maximum likelihood trust and request for manual motion planning, (b) Measure of trust change $c_i$, and (c) Measure of trust feedback $f_i$.}\label{fig:GUI}
\end{figure}

A human subject can choose to collaborate with a robot using a gamepad (see Fig. \ref{fig:humaninput}) and replan a robot path using a mouse (see Fig. \ref{fig:manualpath}). Fig. \ref{fig:GUI} shows the GUI designs used in the simulation for human intervention and collaboration. Fig. \ref{fig:GUI}(a) shows the dynamic evolution of the maximum likelihood trust for all robots and compares the change of trust with a preset threshold. Once the change of maximum likelihood trust exceeds the threshold, the robot with the largest trust increase will request manual motion planning (e.g., Robot 3 requests that the human chooses a path for it as shown in the figure where the threshold is set as 0.07). Fig. \ref{fig:GUI}(b) shows the GUI for measuring the trust change $c_i$ to be used in the calculation of trust belief (\ref{eq:BayesUpdates}) for all robots where ``Lose" corresponds to ``-1", ``Unchange" corresponds to ``0", and ``Gain" corresponds to ``+1". {Table~\ref{tab:one} shows the detailed instructions given to the human operator to intepret the GUI.} This GUI measure is shown to the human subject every  {35} time steps in the simulation. Fig. \ref{fig:GUI}(c) shows the GUI for measuring the trust feedback $f_i$, where ``full distrust", ``medium trust", ``neutral", ``medium trust", and ``full trust" span the spectrum from 0 to 1 and is a continuous scale. This GUI measure is  shown to the human subject every  {100} time steps.  {After 100 time steps, the operator is asked to provide an estimation of their degree of trust towards each robot in the system based on his/her cumulative interaction with the individual robot.} {The gaps of 35 and 100 time steps are selected not only to ensure that the operator will not be overwhelmed, but also to collect sufficient human trust change and trust level data. For example, a gap of 25 seconds is selected for trust evaluation of 1 robot in~\cite{desai2012modeling}.}

{
\begin{table}[ht!]%
\caption{Instructions for GUI in Fig. \ref{fig:GUI}(b)}
\label{tab:one}
\begin{minipage}{1\columnwidth}
\begin{center}
\begin{tabular}{lll}
  \toprule
  YOUR OBSERVATION   & PERFORMANCE CHANGE & ACTION ON GUI\\
  \hline
  A robot detects an obstacle besides it      & Robot performance increases & Choose ``Gain"\\
  path and does not need to change path       & & \\
  A robot reaches a goal destination & Robot performance increases  & Choose ``Gain"\\
  \hline
  A robot senses an obstacle on its path      & Robot performance decreases & Choose ``Lose"\\
  and needs to change path & & \\
  A robot senses another robot on its path      & Robot performance decreases & Choose ``Lose"\\
  and needs to change path & & \\
  An unknown obstacle on human-planned   
   & Robot performance decreases & Choose ``Lose" \\
  path and activate auto. planning & & \\
     \hline
No change & Robot performance unchanged & Choose ``Unchange"\\
  \bottomrule
\end{tabular}
\end{center}
\bigskip\centering
\end{minipage}
\end{table}%
}

\subsection{Overall Simulation for Trust-Based Multi-Robot Motion Planning with a Human-in-the-Loop}\label{sec:sim_human}


{Before implementation of the overall strategy, we run a training session. The purpose is to collect the human input data and performance measures for parameter identification of the trust model. The elapsed time for the training session is 158.59 seconds. After the training session, we run four iterations of the EM algorithm (256.12 seconds) to learn the parameters until the errors converge within the limit $1\times 10^{-10}$.}

We finally show the simulation results of 3 robot symbolic motion planning with a human-in-the-loop. The goal of the human-robot team is to successfully reach each goal while avoiding all collisions, meeting the global specification (\ref{eq:globalSpec}). For this scenario, trust levels are assumed to be equal at the start of the simulation. 
Fig. \ref{fig:manual_path_final} shows the final paths traveled by all robots under the trust-based switching framework. Note that Robot 1 plans a path passing through obstacles for task efficiency. Fig. \ref{fig:fTrust} shows the evolution of mean trust $\bar{T}_i,~i=1,2,3$ over time (Equation (\ref{eq:trust})). 
Figs. \ref{fig:robot1}-\ref{fig:robot3} show the human performance \YueE{$P_{H,i}$}, robot performance \YueE{$P_{R,i}$}, fault $F_i$, trust belief $bel(T_i(t))$, human intervention $m_i$, trust change $c_i$, and trust feedback $f_i$ for Robot 1-3.


\begin{figure}[h!]
	\center
	\includegraphics[width = 0.5\columnwidth]{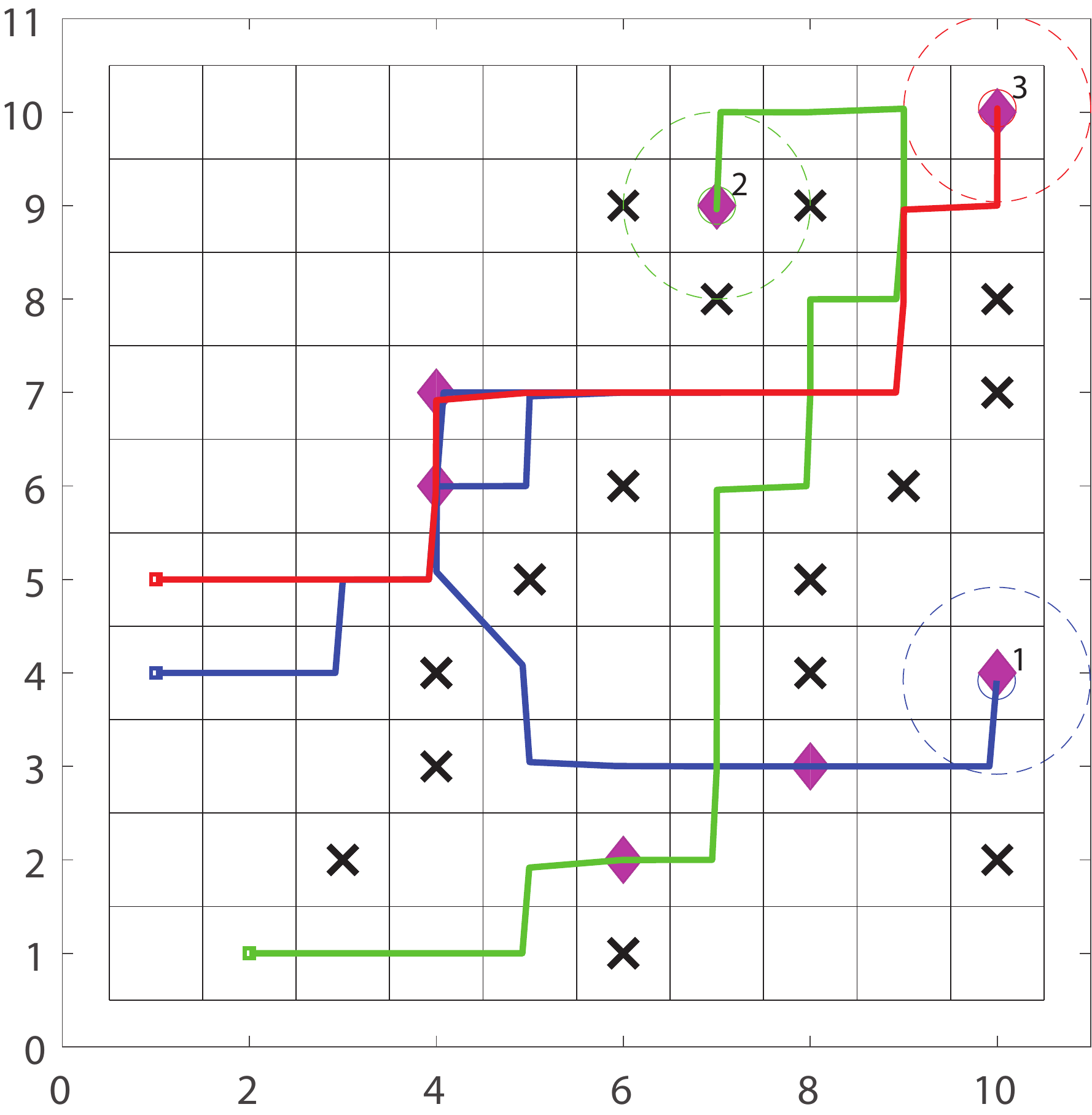}
	\caption{Final paths for 3-robot symbolic motion planning switching between manual and autonomous motion planning mode.}
	\label{fig:manual_path_final}
\end{figure}

\begin{figure}[h!]
	\center
	\includegraphics[width = 1\columnwidth]{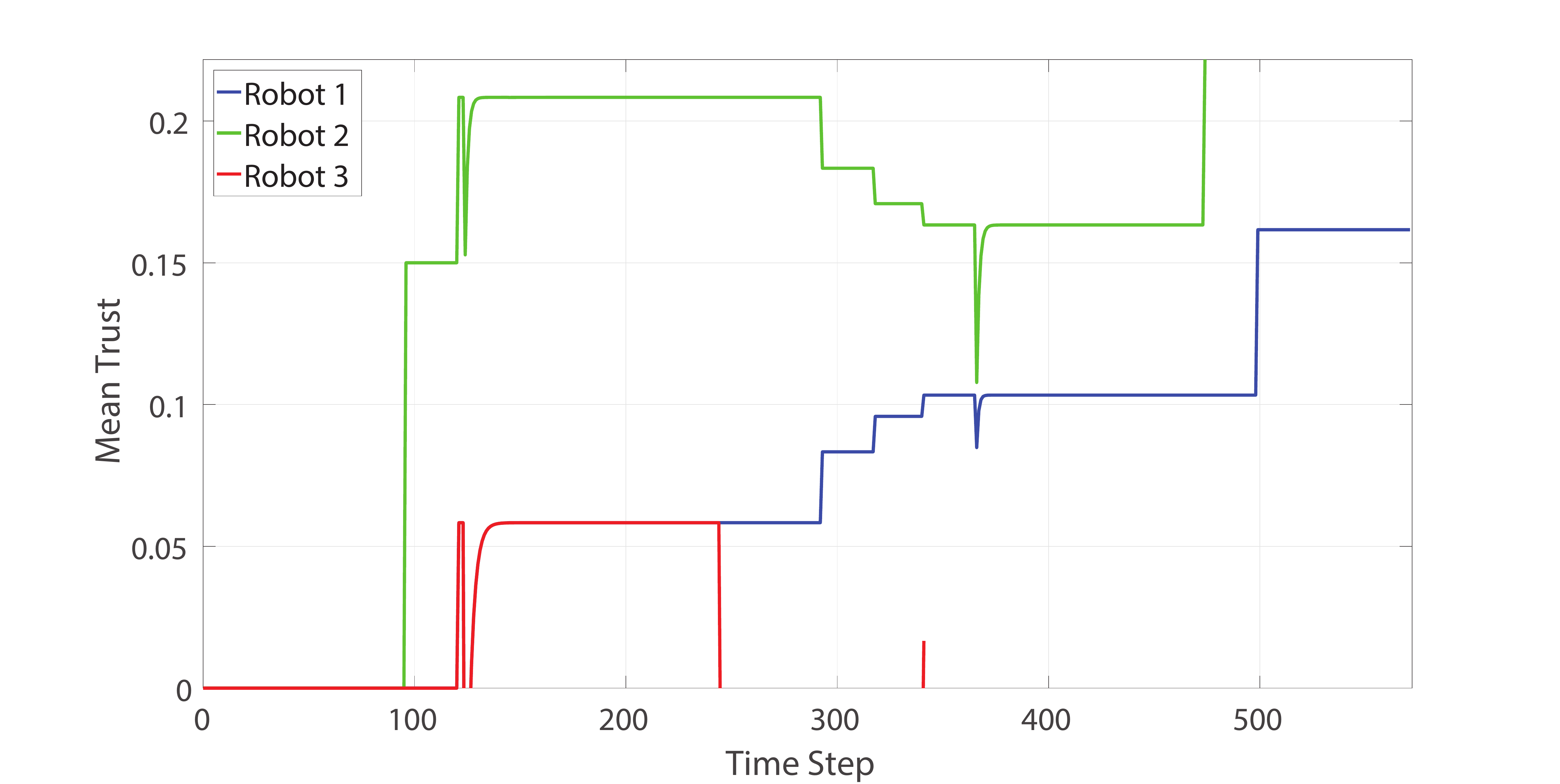}
	\caption{Evolution of mean trust $\bar{T}_i$ for all 3 robots.}
	\label{fig:fTrust}
\end{figure}

\begin{figure}[h!]
	\center
	\includegraphics[width = 1\columnwidth]{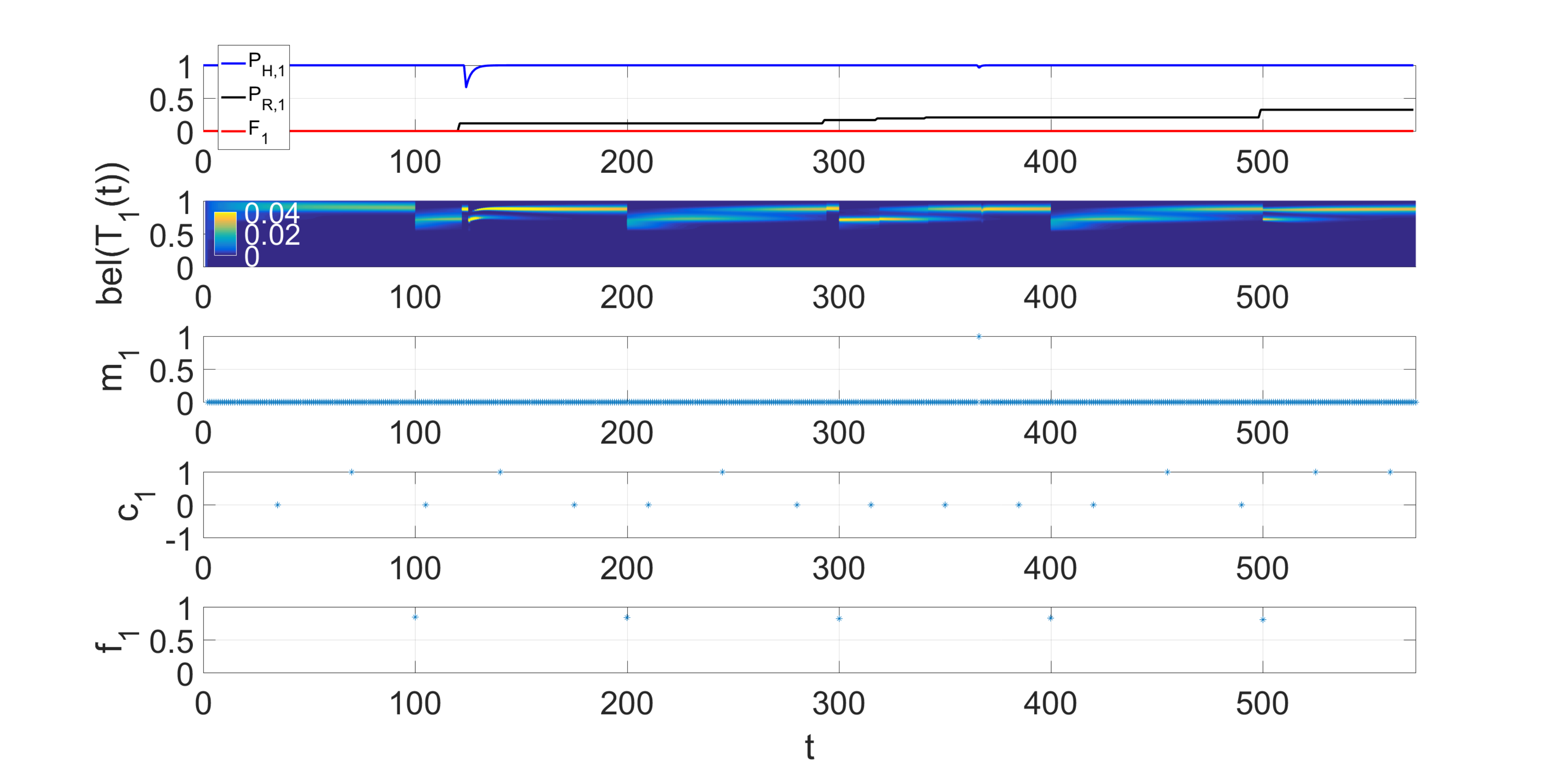}
	\caption{Robot 1: (a) Evolution of human performance \YueE{$P_{H,1}$}, robot performance \YueE{$P_{R,1}$}, and fault $F_1$, (b) trust belief $bel(T_1(t))$, (c) human intervention $m_1$, (d) trust change $c_1$, and (e) trust feedback $f_1$.}
	\label{fig:robot1}
\end{figure}

\begin{figure}[h!]
	\center
	\includegraphics[width = 1\columnwidth]{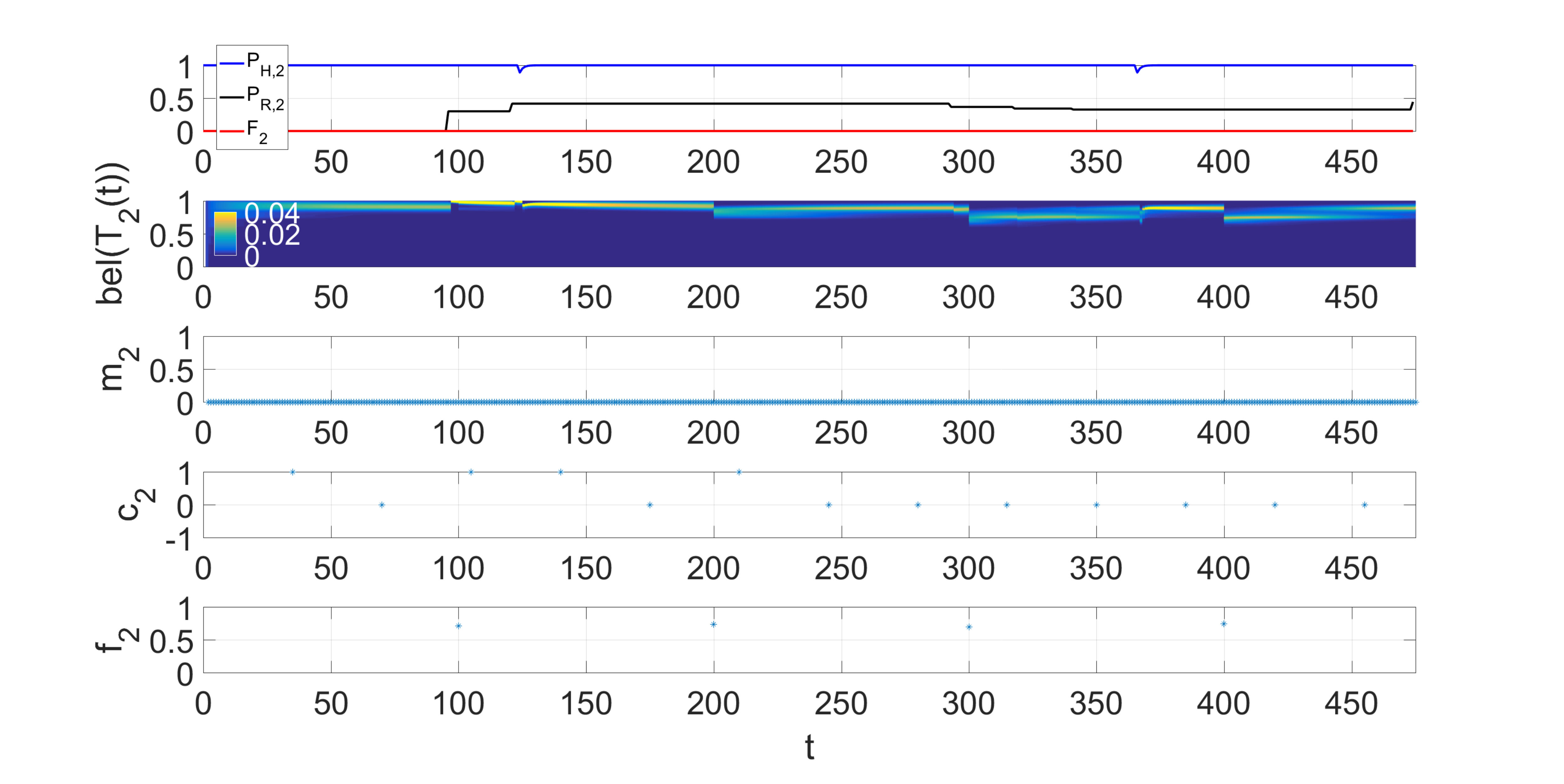}
	\caption{Robot 2: (a) Evolution of human performance \YueE{$P_{H,2}$}, robot performance \YueE{$P_{R,2}$}, and fault $F_2$, (b) trust belief $bel(T_2(t))$, (c) human intervention $m_2$, (d) trust change $c_2$, and (e) trust feedback $f_2$.}
	\label{fig:robot2}
\end{figure}

\begin{figure}[h!]
	\center
	\includegraphics[width = 1\columnwidth]{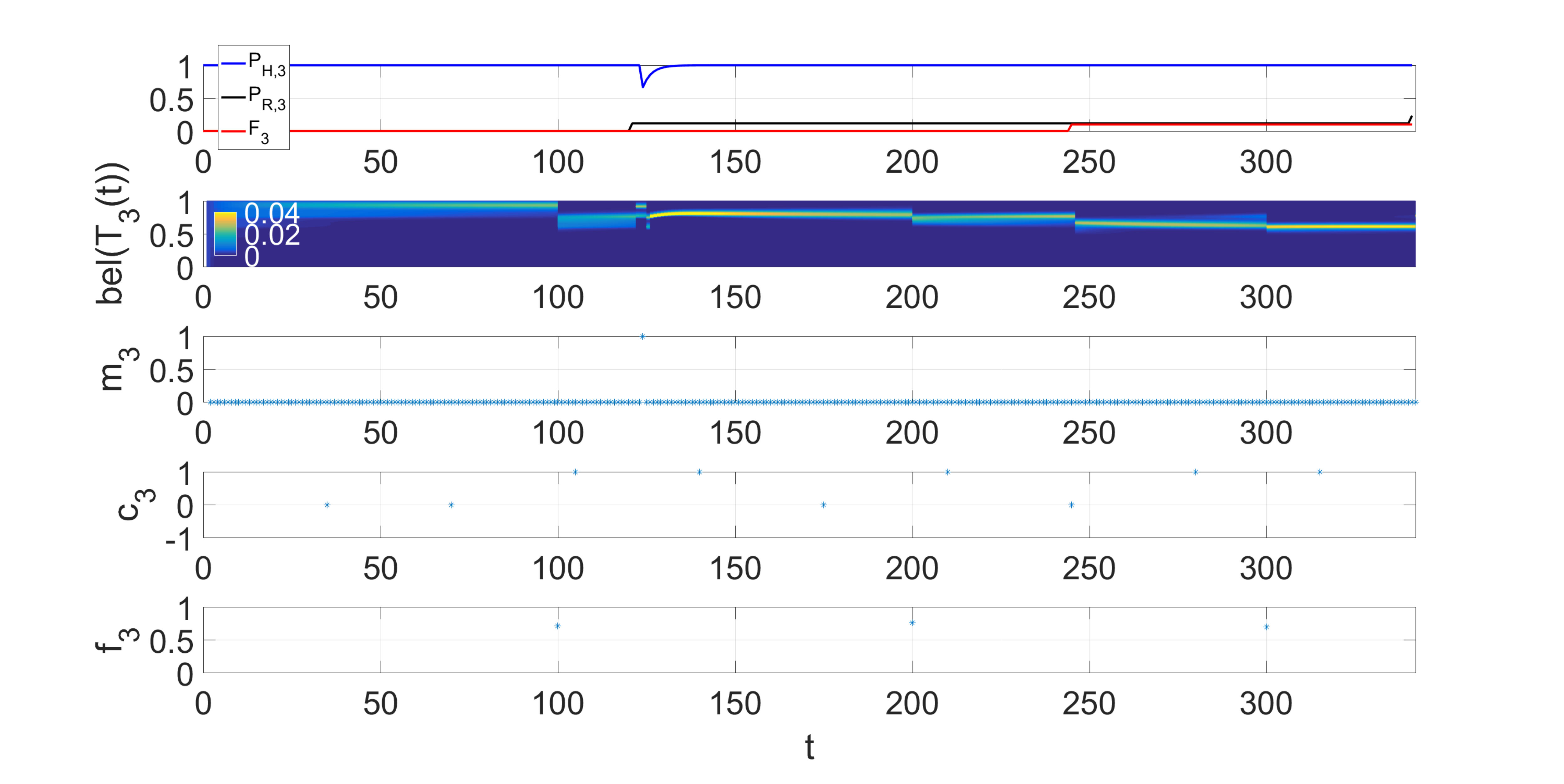}
	\caption{Robot 3: (a) Evolution of human performance \YueE{$P_{H,3}$}, robot performance \YueE{$P_{R,3}$}, and fault $F_3$, (b) trust belief $bel(T_3(t))$, (c) human intervention $m_3$, (d) trust change $c_3$, and (e) trust feedback $f_3$.}
	\label{fig:robot3}
\end{figure}

Initially, Robot 1 is assigned to goals (8,3) and (4,6), Robot 2 to goals (6,2), (7,9), and Robot 3 to goals (10,4), (4,7), (10,10).
At the beginning, Robot 1 senses an obstacle on its path and Robot 2 senses an obstacle on its left side. Therefore, in the first trust change question (see Fig.~\ref{fig:GUI}(b)), Robot 1's trust decreases by 1 and Robot 2's trust increases by 1. That is to say, in the trust question, when a robot senses an obstacle on its path, the human subject should choose ``Gain". When a robot senses an obstacle on its path or meets another robot and hence changes its path, one should choose ``Lose". Otherwise, one should keep choosing ``Unchange".

When all three robots reach their first goals at time step  {$t=124$}, the trust increase of Robot 3 goes beyond the threshold (set as 0.05 here). Hence, Robot 3 requests the human to intervene (see Fig.~\ref{fig:GUI}(a)) and the human plans a new path for Robot 3 (as shown in Fig.~\ref{fig:robot3}(c) with $m_3( {t=124})=1$). However, there is an unknown obstacle in the path. Once this unknown obstacle is sensed, a fault occurs at $t=246$ (Fig.~\ref{fig:robot3}(a)) and Robot 3's trust drops (Fig.~\ref{fig:robot3}(b)). Robot 3 then gives its current goal (10,4) to Robot 1 (whose trust change is the highest at this moment). 
Robot 3 continues to its last left goal (10,10) and stops and stays there at  {$t=342$}.

For Robot 1, as shown in Fig.~\ref{fig:robot1}, $P_{R,1}$ increases at time step $t=122$, and then \YueE{$P_{H,1}$} drops at time step $t=126$ with the trust belief increases and drops, correspondingly. At  {$t=366$}, Robot 1 meets obstacles, leads to trust increase, and requests for human intervention. The human subject plans a path passing through obstacles and reaches its left goals (see Fig. \ref{fig:manual_path_final}). 
At  {$t=500$}, \YueE{$P_{R,1}$} and trust increases accordingly. Robot 1 finally stops at \YueE{$t=572$} after reaching all the goals including the newly assigned goal (10,4). 

Robot 2 is in autonomous mode throughout the simulation. As shown in Fig.~\ref{fig:robot2}, at time steps  {125 and 367,} due to the human interaction with the other two robots, \YueE{$P_{H,2}$} drops and trust drops correspondingly. At time steps 97 and 122, Robot 2 senses obstacles. \YueE{$P_{R,2}$} increases as a result and trust increases accordingly. At $t=319$, \YueE{$P_{R,2}$} drops and trust drops correspondingly. The drop is because Robot 2 meets an obstacle on its planned path. Robot 2 stops at \YueE{$t=475$} after reaching all of its assigned goals.

\section{Conclusions}
\label{sec:conclusion}

In this paper, we have explored methods to address scalability, safety, performance, and adaptability of symbolic motion planning for multi-robot systems that interact with a human operator, with interactions affected by human trust.
A quantitative, dynamic, and probabilistic model is developed to compute human trust in a robot during a collaborative ISR task. 
Scalability is addressed by decomposing portions of the global specification related to goal reachability and obstacle avoidance, while portions that require inter-robot collision avoidance are addressed through a protocol that relies on local communication and control to modify plans as needed during execution.
Trust affects this decomposition adaptively, with more trusted robots assigned more destinations, and with trust decreasing as robots generate faults by coming too close to obstacles.
In addition to implementing methods for obstacle and collision avoidance, safety versus efficiency of planning is addressed by switching between a safe but conservative robot motion planning mode that avoids
obstacles based on an overapproximation of the environment and a riskier but more efficient human planning mode that allows paths between obstacles, with switching mediated by human trust. Finally, deadlock- and livelock-free algorithms are proposed to guarantee the reachability of all goal destinations with a human-in-the-loop.

In the future, there are several other areas that could be explored.
The computational trust model we use here assumes trust in a specific robot evolves independently of all other robots,
when in actuality, evolution of trust for each robot might be interdependent \cite{kellerRice-09}. Other human-to-robot trust models could be used in the same framework.
Moreover, we have assumed that obstacles are placed such that all destinations are reachable by all robots, which would not be true in all environments and would require more careful assignment of destinations.
In some cases, certain destinations might not be reachable by any robot, which would require revising the specification.
 {The simulations in Section \ref{sec:simulation} present a simple environment to demonstrate the overall strategy. However, in real-world applications (e.g. automated storage and retrieval robotic systems) the number of robots, goals and obstacles can be much larger. Correspondingly, human's cognitive workload will increase as well and overload will negatively affect the user performance. To further quantify the workload and the corresponding effects on user performance, a proper workload model will be adopted in our future work.}
\begin{acks}
{The authors would like to thank Dr. Anqi Xu of McGill University for the discussion of the OPTIMo trust model.}
\end{acks}

\bibliographystyle{ACM-Reference-Format}
\bibliography{bibliography}



\medskip

%
%
%
%

\end{document}